\documentclass{article}

\usepackage{times}
\usepackage{graphicx} 
\usepackage{subfigure} 
\usepackage{natbib}

\usepackage[accepted]{icml2016}
\usepackage{amsmath,amssymb} 

\usepackage{multirow}
\usepackage{amsthm}
\newtheorem{lemma}{Lemma}

\usepackage{array}

\newlength\savewidth

\newcolumntype{C}{>{\centering\arraybackslash} m{3.0cm} }
\newcolumntype{Z}{>{\centering\arraybackslash} m{3.5cm} }
\newcolumntype{Q}{>{\centering\arraybackslash} m{7.5cm} }
\newcolumntype{Y}{>{\centering\arraybackslash} m{4.5cm} }
\newcolumntype{D}{>{\centering\arraybackslash} m{2.5cm} }
\newcolumntype{X}{>{\centering\arraybackslash} m{6.5cm} }
\newcolumntype{W}{>{\centering\arraybackslash} m{12.0cm} }
\newcolumntype{B}{>{\centering\arraybackslash} m{5.0cm} }

\icmltitlerunning{Towards Robust CNNs Using SAFs}

\begin{document} 

\twocolumn[
\icmltitle{Suppressing the Unusual: towards Robust CNNs using Symmetric Activation Functions}

\icmlauthor{Qiyang Zhao}{zhaoqy@buaa.edu.cn}
\icmladdress{Beihang University, 37 Xueyuan Rd., Beijing 100191, China}
\icmlauthor{Lewis D Griffin}{l.griffin@cs.ucl.ac.uk}
\icmladdress{University College London, Gower St., London, WC1E 6BT, UK}

\icmlkeywords{adversarial samples; nonsense samples; SAF; CNN}

\vskip 0.3in
]

\begin{abstract} 
Many deep Convolutional Neural Networks (CNN) make incorrect predictions on adversarial samples obtained by imperceptible perturbations of clean samples. We hypothesize that this is caused by a failure to suppress unusual signals within network layers. As remedy we propose the use of Symmetric Activation Functions (SAF) in non-linear signal transducer units. These units suppress signals of exceptional magnitude. We prove that SAF networks can perform classification tasks to arbitrary precision in a simplified situation. In practice, rather than use SAFs alone, we add them into CNNs to improve their robustness. The modified CNNs can be easily trained using popular strategies with the moderate training load. Our experiments on MNIST and CIFAR-10 show that the modified CNNs perform similarly to plain ones on clean samples, and are remarkably more robust against adversarial and nonsense samples.
\end{abstract} 

\section{Introduction}

Although deep CNNs have delivered state-of-the-art performance on several major computer vision challenges \cite{Krizhevsky}\cite{He}, they have some pooly understood aspects. In particular, it was shown in \cite{Szegedy} that one can easily construct an \emph{adversarial sample} by imperceptible perturbation of any \emph{clean sample} using the 
box-constrained limited-memory BFGS method \cite{Andrew}. These samples are perceptually similar to the original ones, but are predicted to be of different categories, and even generalize well across different CNN models and train sets \cite{Szegedy}. In \cite{Goodfellow}, a simpler method, fast gradient sign (FGS), is proposed to compute adversarial samples effectively. Recently, a more complicated method is proposed to construct an adversarial sample with both the prediction and internal CNN features being similar to another arbitrary sample \cite{Sabour}. \emph{Nonsense samples} (called rubbish samples in \cite{Goodfellow}) are another challenge to the robustness of CNNs. These are generated from noise or arbitrary images using FGS-like methods \cite{Goodfellow}\cite{Nguyen}. Nonsense samples are totally unrecognizable but popular CNNs typically predict with high confidence that they belong to some category. All these references argue that adversarial and nonsense samples arise from some unknown flaw of popular deep CNNs, either in their structure or training methods \cite{Szegedy}\cite{Goodfellow}\cite{Sabour}.

\begin{figure}[htbp]
	\centering
	\begin{tabular}{ZZ}
		\includegraphics[width=3.30cm]{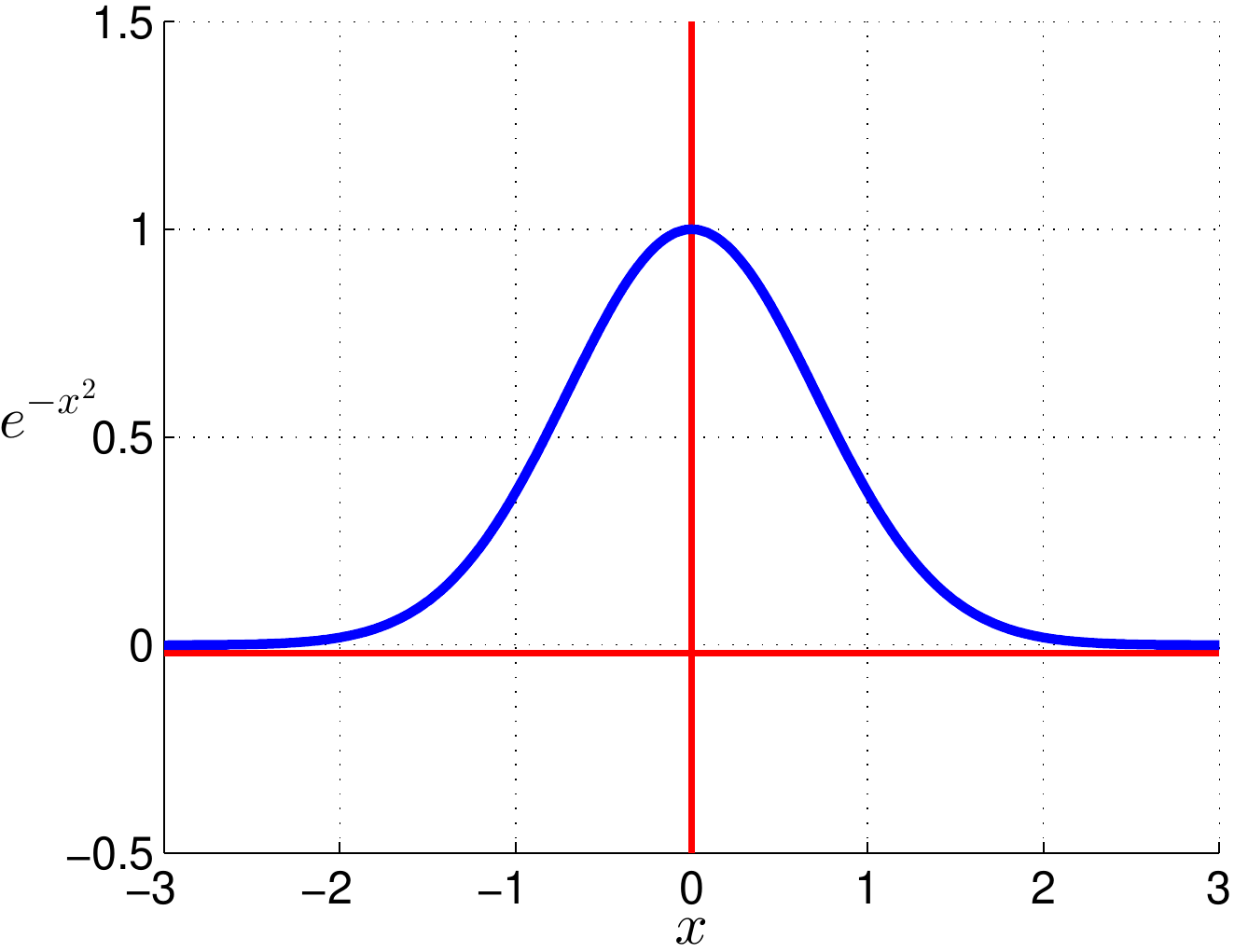}&
		\includegraphics[width=3.40cm]{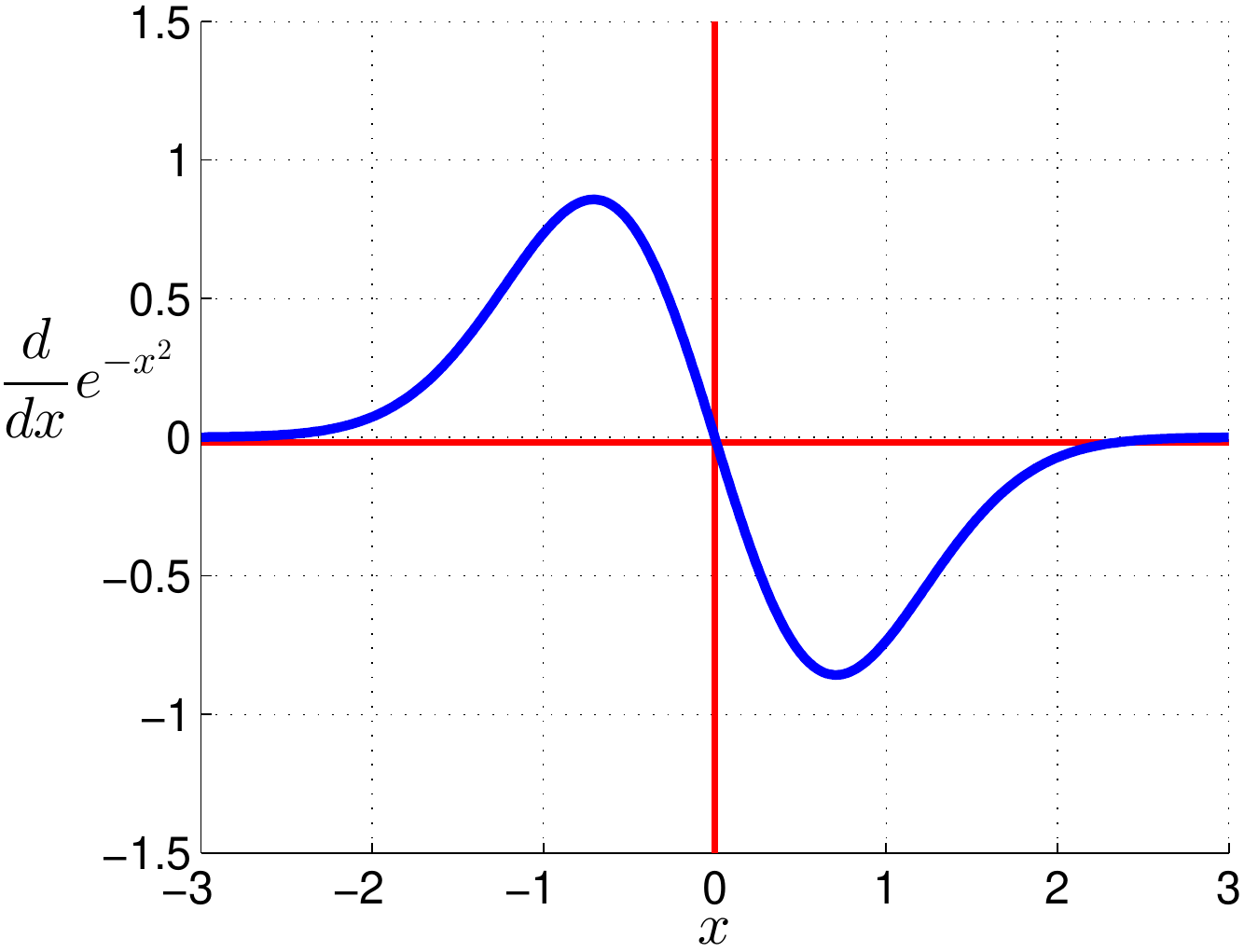}\\
		\ \ \ \ (a) & \ \ \ \ (b) \\
		\includegraphics[width=3.30cm]{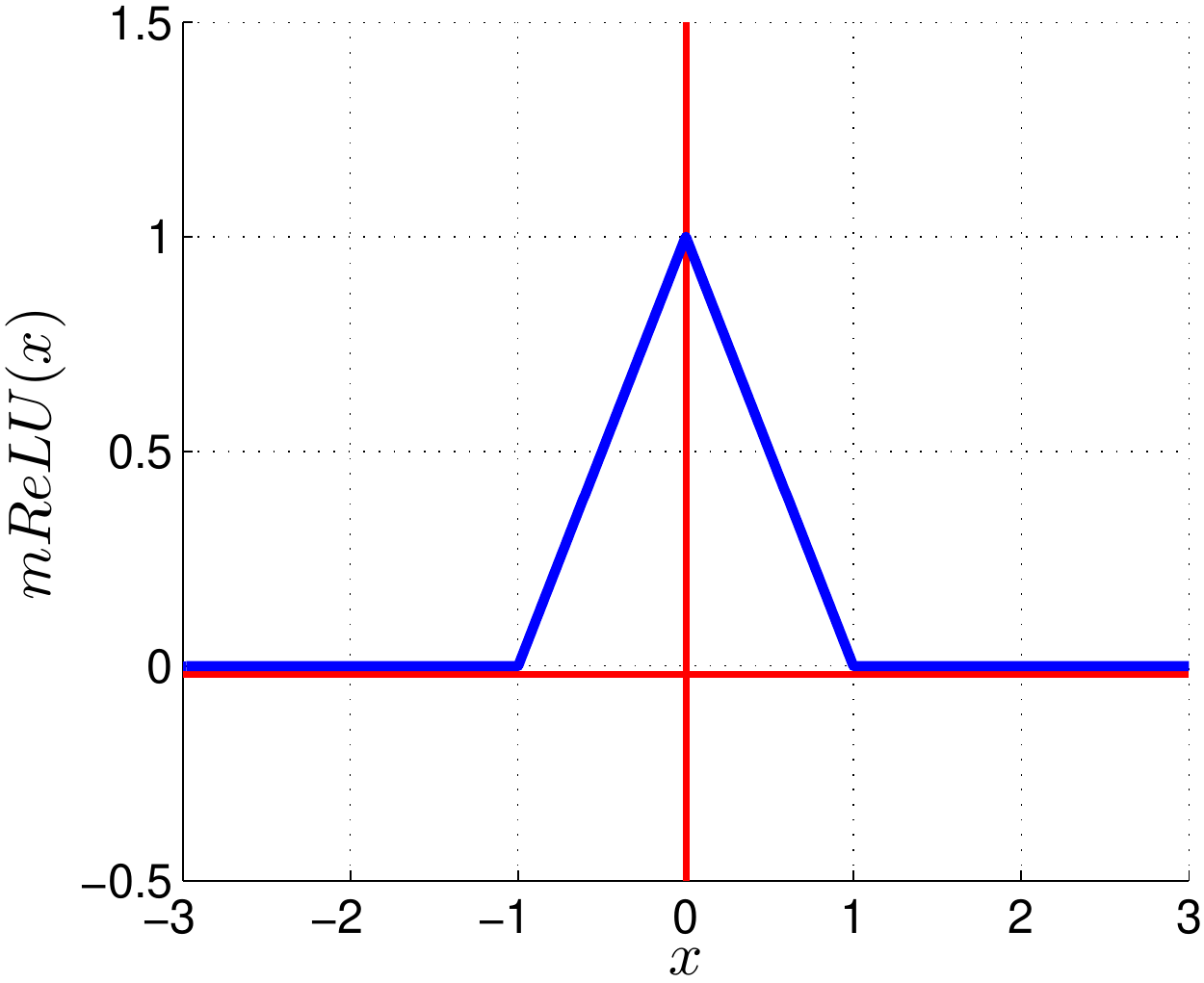}&
		\includegraphics[width=3.50cm]{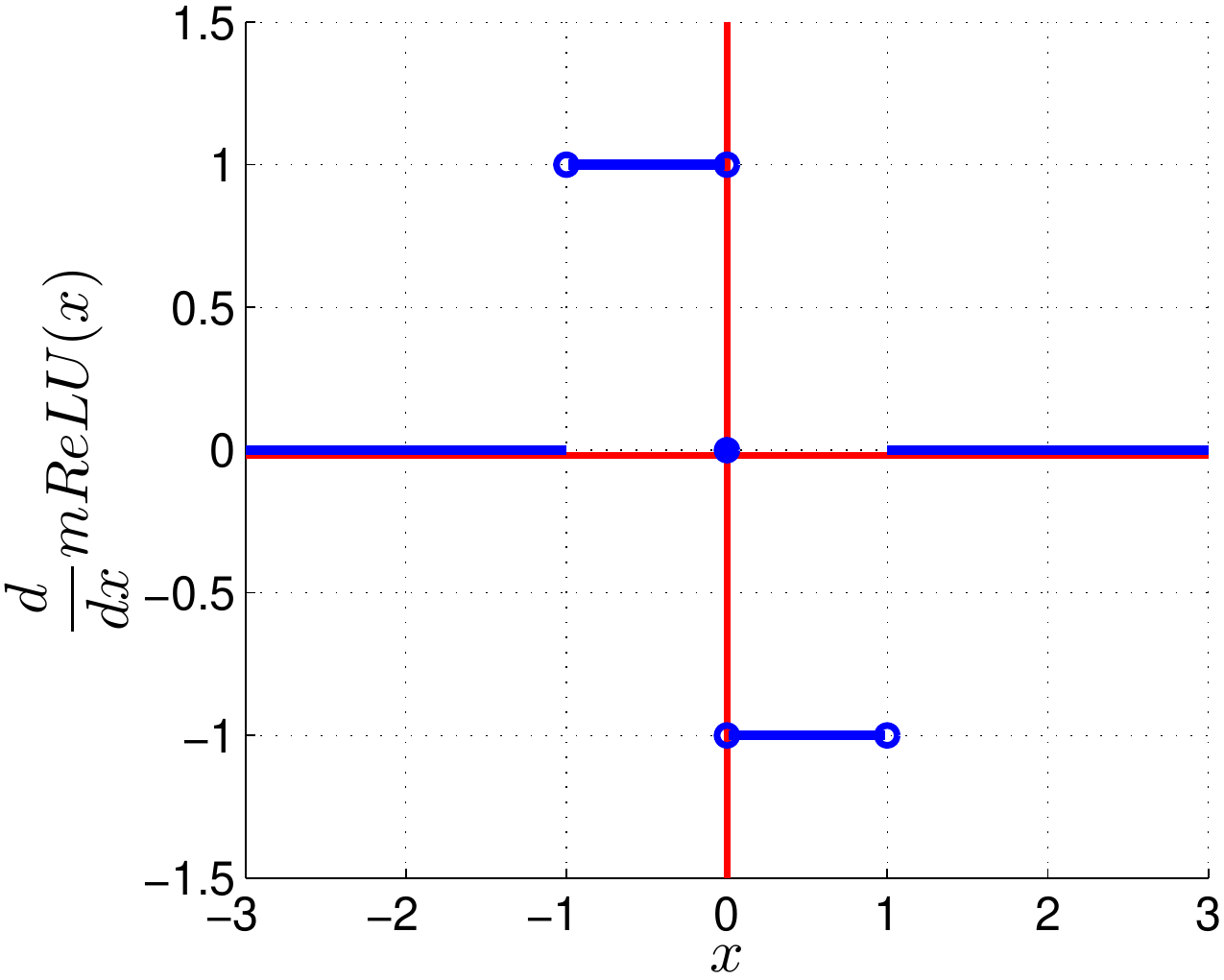}\\
		\ \ \ \ (c) & \ \ \ \ \ \ (d) \\
	\end{tabular}
	\caption{Symmetric activation functions and their derivatives . (a)(b) 1-D RBF and its derivative; (c)(d) mReLU and its derivative. Please see Sec. \ref{sec:saf} for details.}
	\label{fig:symfunc_der}
\end{figure}

\begin{figure}[htbp]
	\centering
	\begin{tabular}{Z|Z}
		\includegraphics[width=3.0cm]{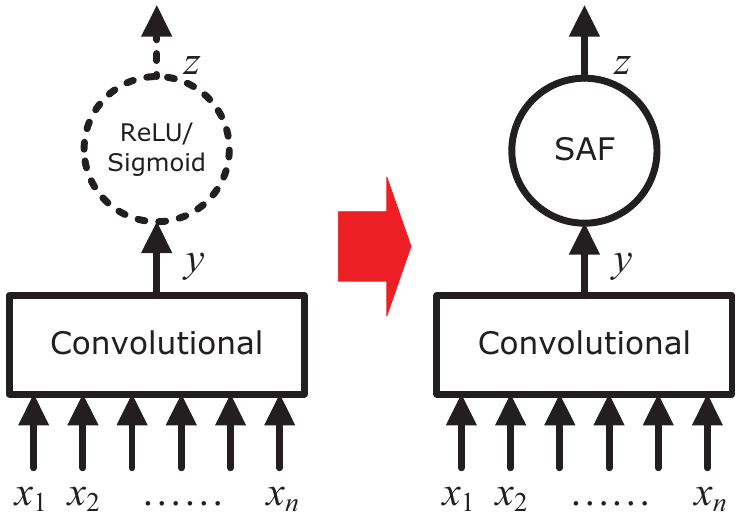}&
		\includegraphics[width=2.3cm]{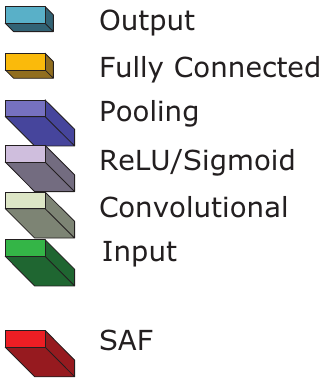}\\
		(a) & \\
		\includegraphics[width=3.0cm]{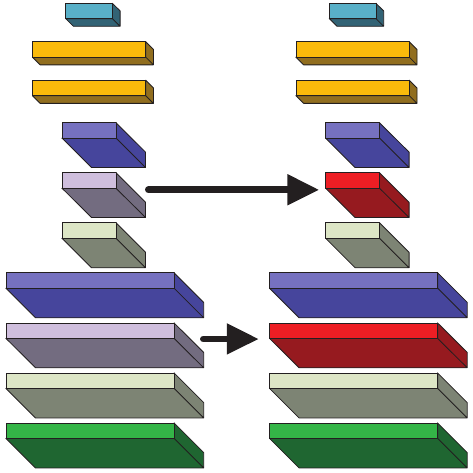}&
		\includegraphics[width=3.0cm]{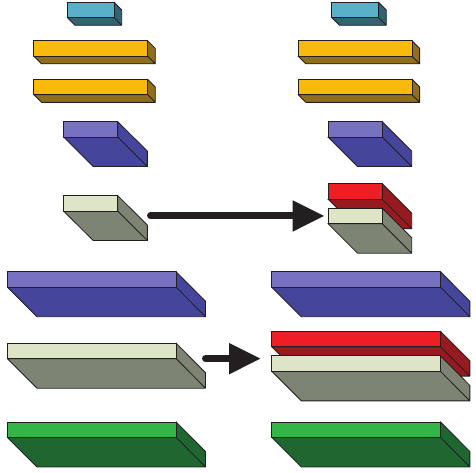}\\
		(b) & (c)
	\end{tabular}
	\caption{Building robust CNNs using SAF units. (a) the basic block of a SAF unit. (b) replacing ReLU/sigmoid layers with SAF layers. (c) inserting SAF layers into a CNN.}
	\label{fig:building_RobustCNN}
\end{figure}

A proposed solution to the problem is \emph{adversarial training} \cite{Szegedy}\cite{Goodfellow}: adversarial samples are constructed according to the current model, then utilized in subsequent training epochs. CNNs so trained will be vulnerable to new adversarial samples; but repeated rounds of adversarial training eventually lead to improved robustness, at the price of an increased training load. Unfortunately, the final error rates on the newest adversarial samples, after this scheme, remain large \cite{Goodfellow}. It is interesting that adversarial training can be understood as a new regularization method for training CNNs; for example in \cite{Miyato}, virtual adversarial samples are constructed to force the model distributions of a supervised learning process to be smooth. Adversarial samples can also be used to establish new unsupervised and/or semi-supervised representation learning methods \cite{Radford} \cite{Springenberg} or autoencoders \cite{Makhzani}.

A different approach to the adversary problem is to modify the structure of the CNN. Autoencoders can be added to denoise input samples and improve the local stability of feature spaces in individual layers \cite{Gu}, where the contractive penalty terms are designed to punish unexpected feature changes caused by small perturbations. The idea is similar to \cite{Miyato}. However this method cannot achieve good robustness on moderately large perturbations while keeping high accuracies on clean samples \cite{Gu}. Another interesting idea, proposed in \cite{Chalupka}, is to identify features that are causally related, rather than merely correlated, with the categories. It is thus possible to predict the semantic categories according to causal features, so to improve the robustness. However it would be a large burden to integrate this idea with popular CNN models.

In this paper we propose a solution to the adversary problem that does not increase much training load, nor increase the complexity of the network. The solution is to use symmetric activation function (SAF) units. These are non-linear transduction units that suppress exceptional signals. We incorporate SAFs into CNNs by addition or replacement of layers of standard models. The resulting CNNs are easily trained using standard approaches and give low error rates on clean and perturbed samples.

\section{Motivation}
\label{sec:motivation}

Unlike CNNs and linear classification models, local models \cite{Moody}\cite{Jacobs} based on radial basis functions (RBF) \cite{Broomhead} are intrinsically immune to abnormal samples which are far from high-density regions: they will always output low confidence scores for such. Furthermore, the RBFs are robust against tiny perturbations of the input. This suggests the possibility of improving the robustness of CNNs by using RBF units.

RBFs have been used in several ways within networks, e.g. as a hidden layer in mixture of experts \cite{Jacobs}, and as the top layer units in LeNet-5 \cite{LeCun}.
Only a few implementations attempt multiple RBF layers \cite{Craddock}\cite{Bohte}\cite{Mrazova}. The reason for avoiding multiple layers is the high-dimensionality of the parameter space of each RBF unit, together with the need for many RBF units. This high-dimensionality makes it hard to train deep networks involving RBF units to achieve acceptable accuracies \cite{Goodfellow}. Some have argued that more powerful optimization methods are needed to make this work \cite{Goodfellow}.

We have considered whether it is feasible to use 1-D RBFs instead of high-dimension ones. To assess this, we run a LeNet-5 model \cite{LeCun} (see Sec. \ref{sec:exp_mnist} for its details) on the MNIST test set, and show the histograms and joint distributions of features from the second convolutional layer in Fig. \ref{fig:dist_mnist_plainLeNet}. The empirical distributions of these features are approximately Gaussian. In our experiments, we find it is effective to distinguish abnormal samples from normal ones by exploiting this observation. In the second row of Fig. \ref{fig:dist_mnist_plainLeNet}, we show the feature pair values of nonsense samples of pure Gaussian noise in red squares. Although the LeNet-5 model incorrectly assigns nonsense samples to meaningful categories with high confidence ($\sim \! 99\%$), as the figure shows, their features usually deviate from the high-density regions of clean samples. The higher the layer the more distinct the deviation.

We also observe that features are not strongly correlated, as shown in Fig. \ref{fig:corr_mnist_plainLeNet}. Therefore it is viable to use 1-D Gaussian distributions instead of high-dimension ones to model the feature statistics of clean samples. Such 1-D RBFs can be used within networks, as multi-dimensional RBFs have been. The potential advantage is obvious: it may be possible to use much fewer 1-D RBF units than high-dimensional ones, and therefore with fewer parameters to train. It should be pointed out that we do not aim to approximate a traditional RBF network of only a single hidden layer using 1-D RBFs. Instead, we propose to integrate 1-D RBF units with popular CNN models. The integration is neither a pure `local' model such as RBF networks nor a pure `distributed' model such as CNNs. It models diverse features `locally' with 1-D RBFs but combines them in a `distributed' way to give the final representation.

\begin{figure}[htbp]
	\centering
	\begin{tabular}{ZZ}
		\includegraphics[width=3.0cm]{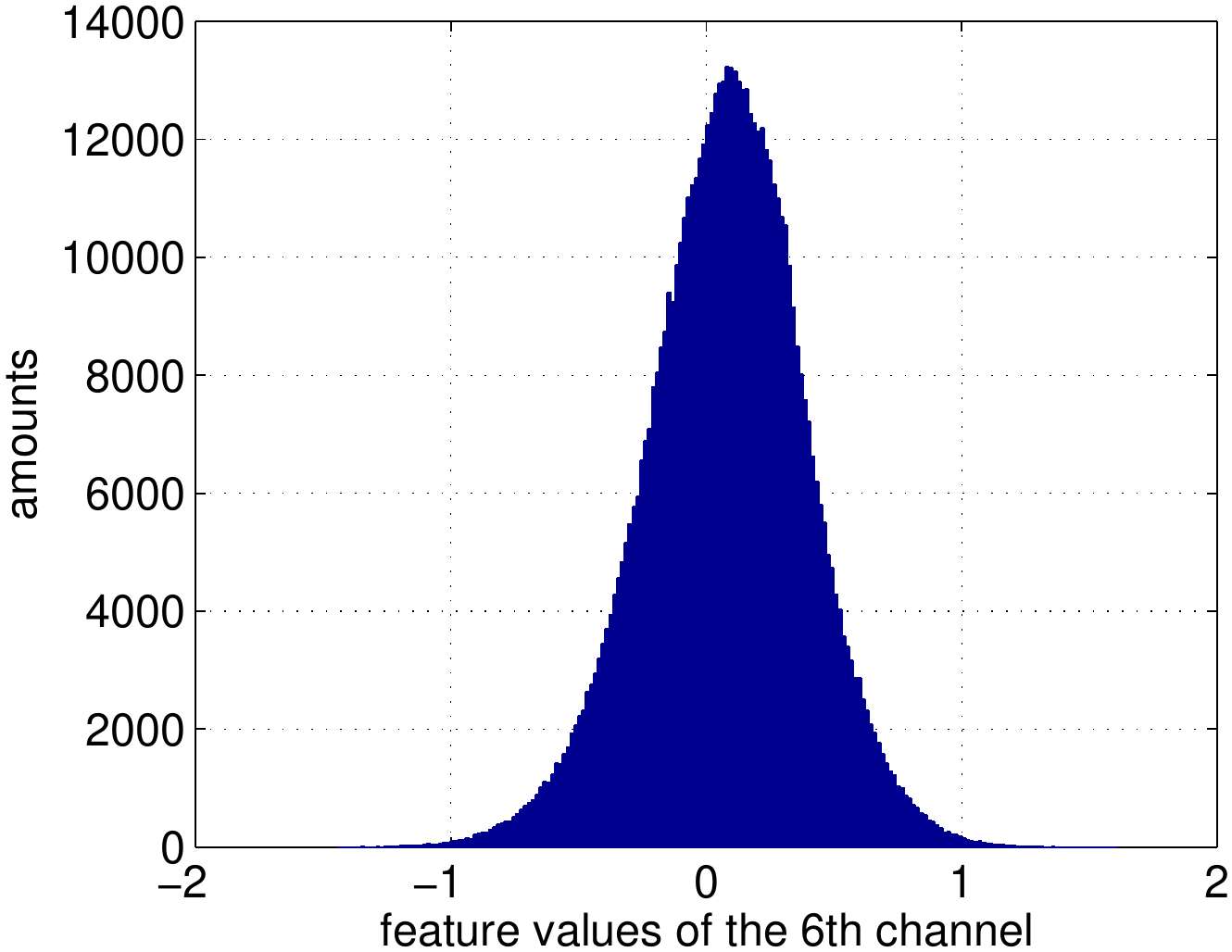}&
		\includegraphics[width=3.0cm]{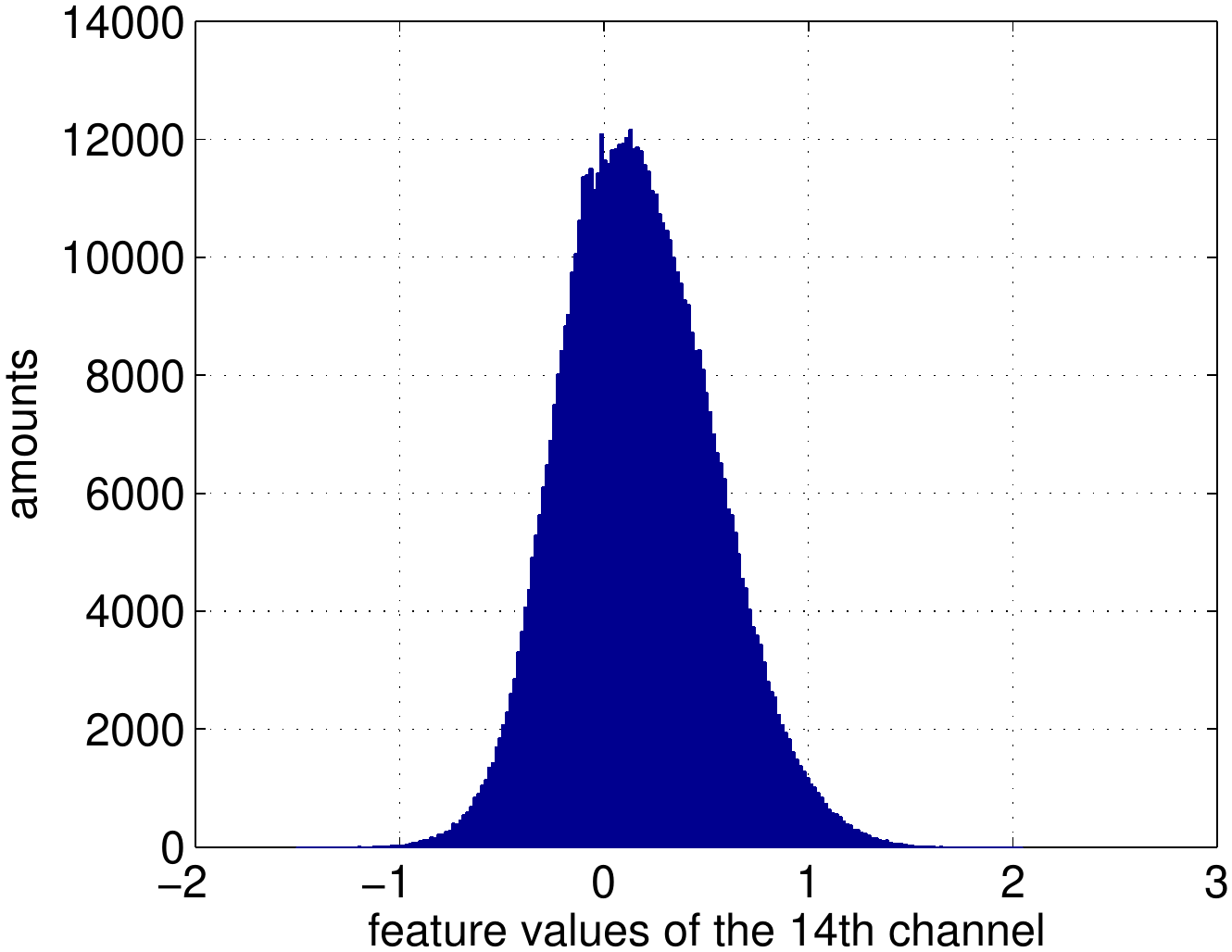}\\
		\includegraphics[width=3.0cm]{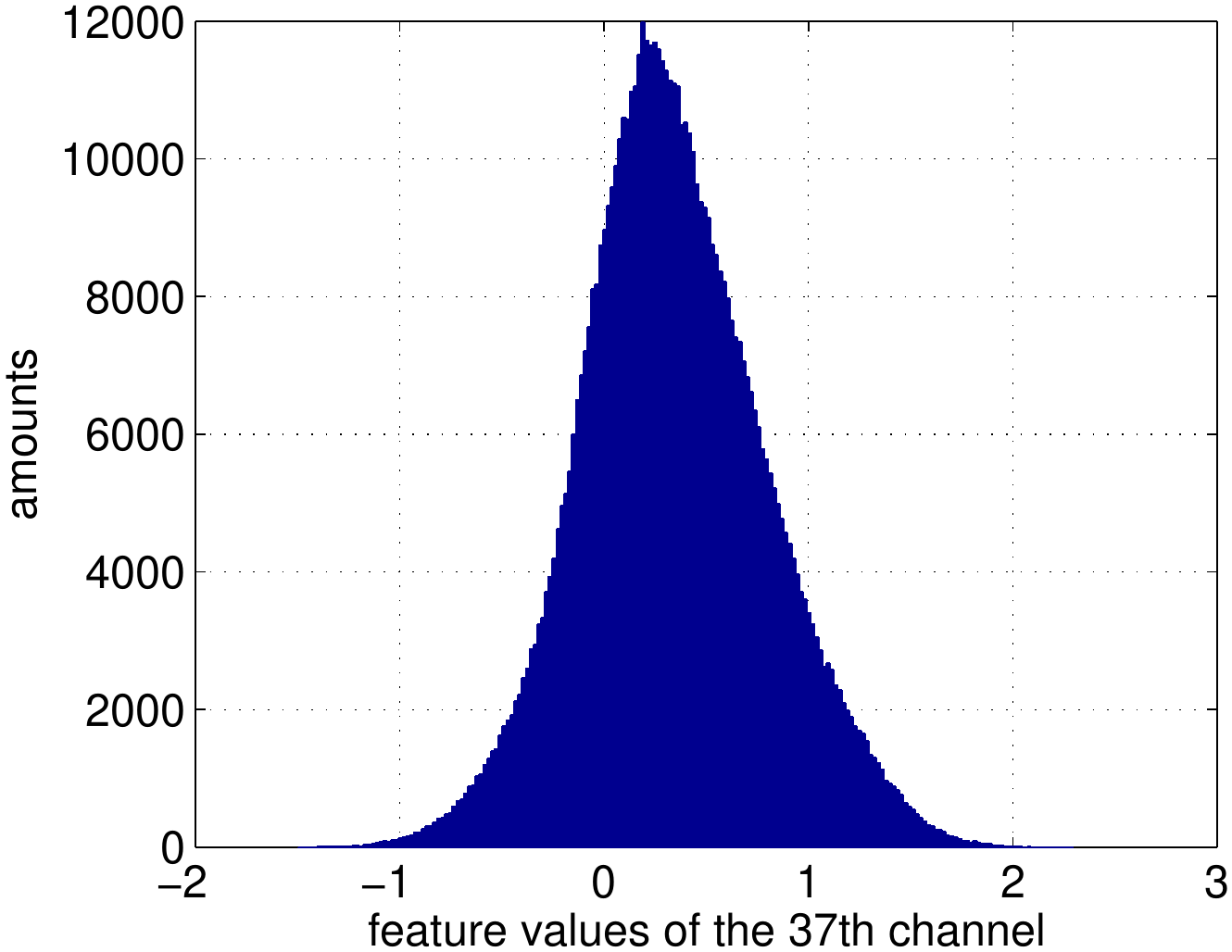}&
		\includegraphics[width=3.0cm]{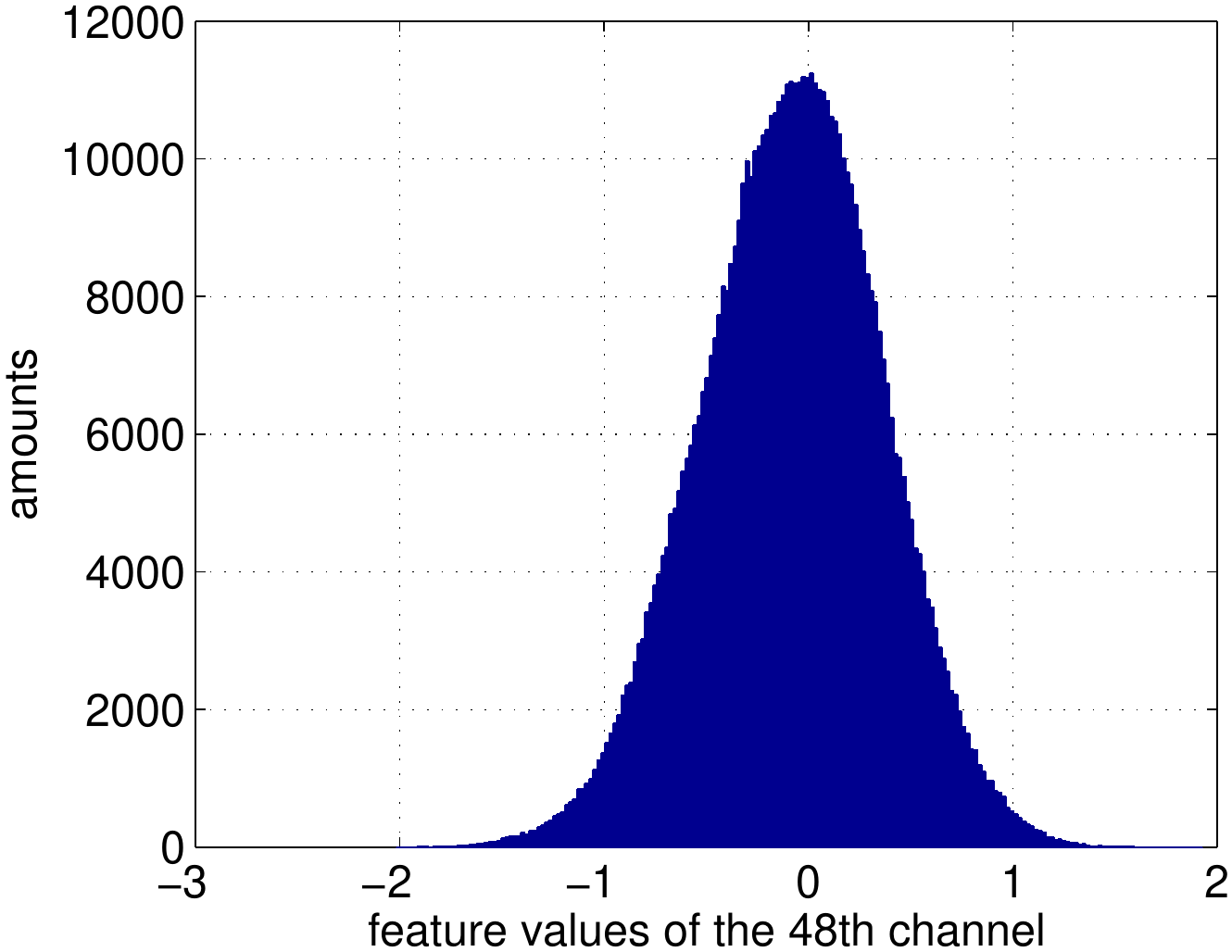}\\
		\includegraphics[width=3.0cm]{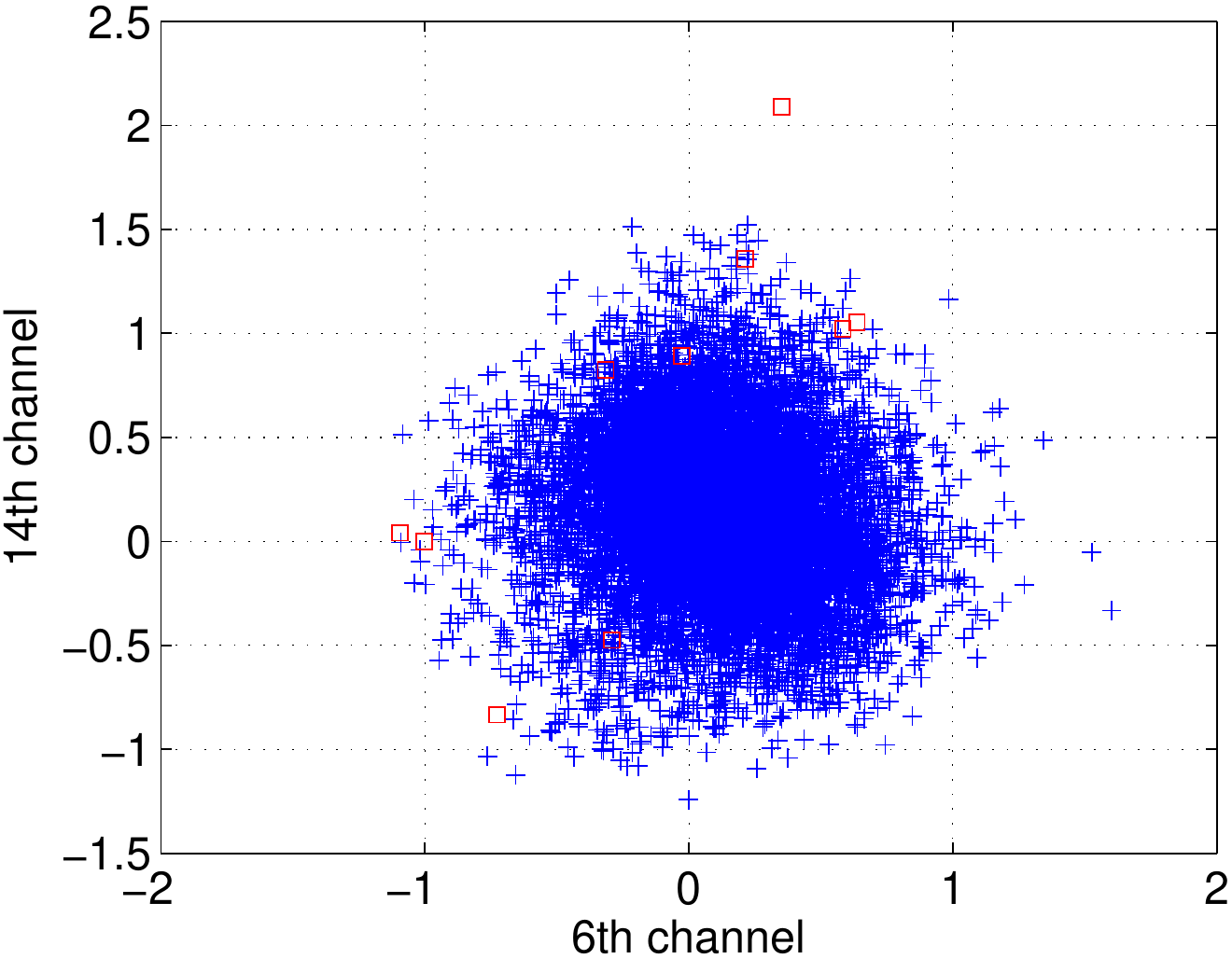}&
		\includegraphics[width=3.0cm]{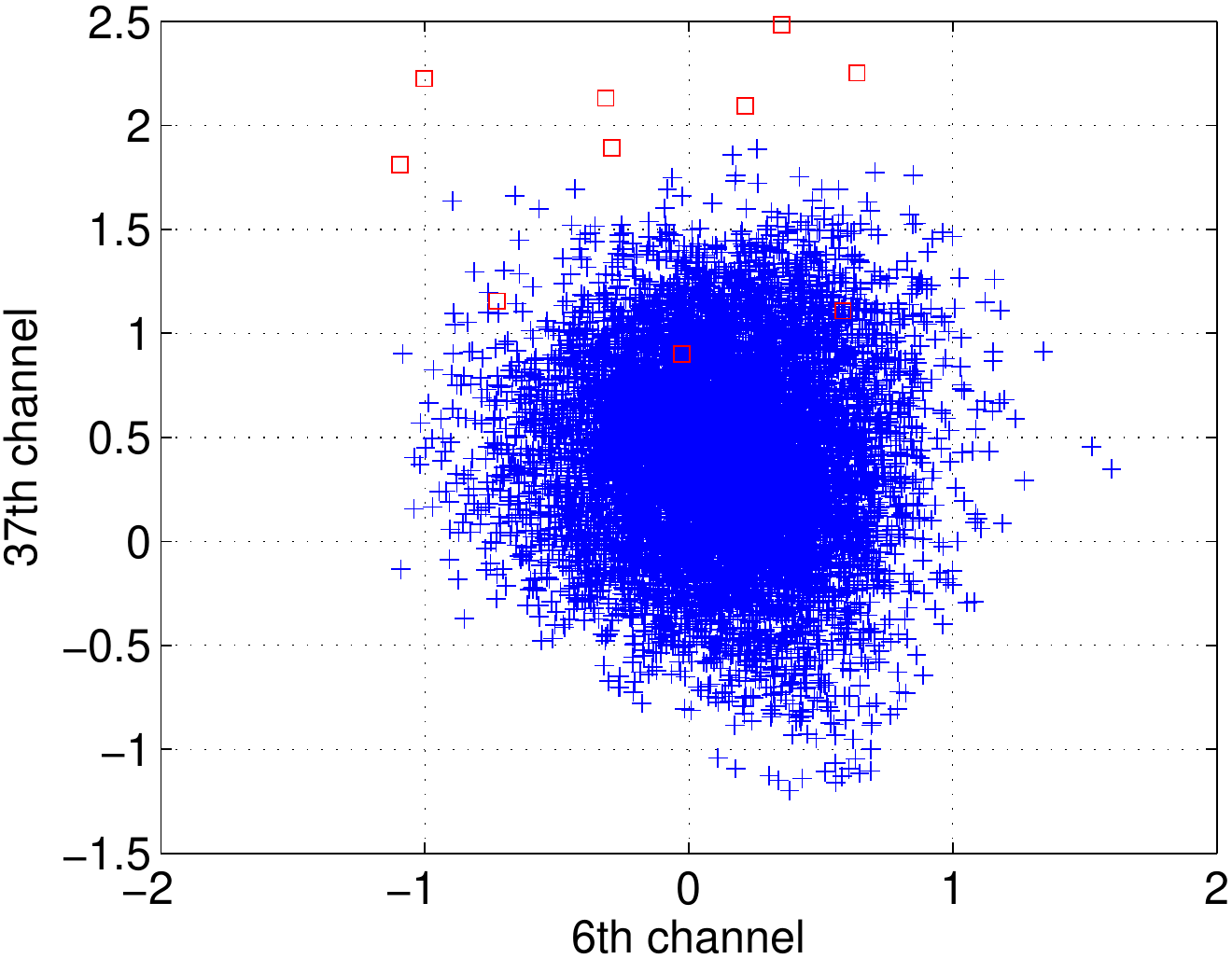}\\
		\ \includegraphics[width=2.90cm]{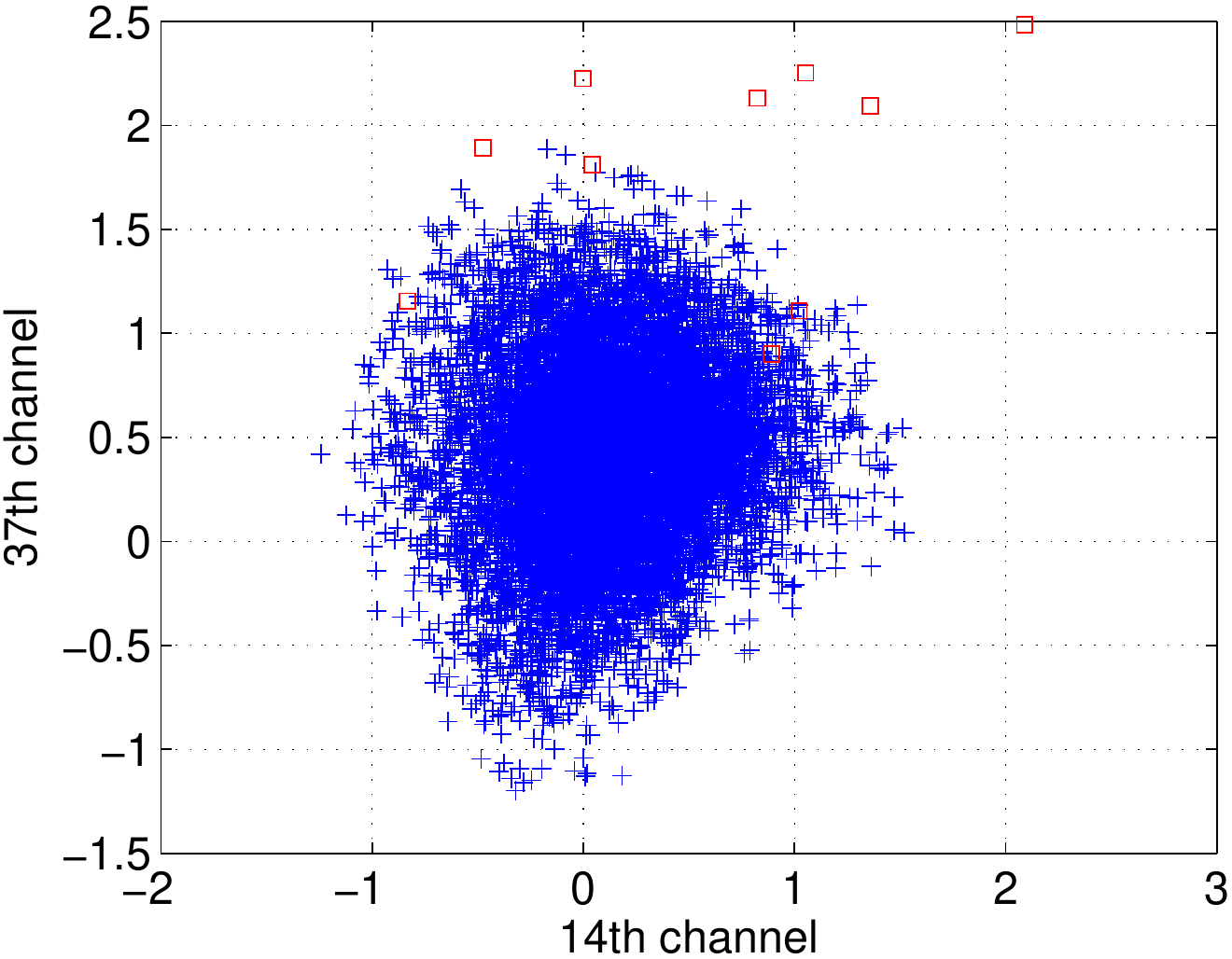}&
		\ \includegraphics[width=2.90cm]{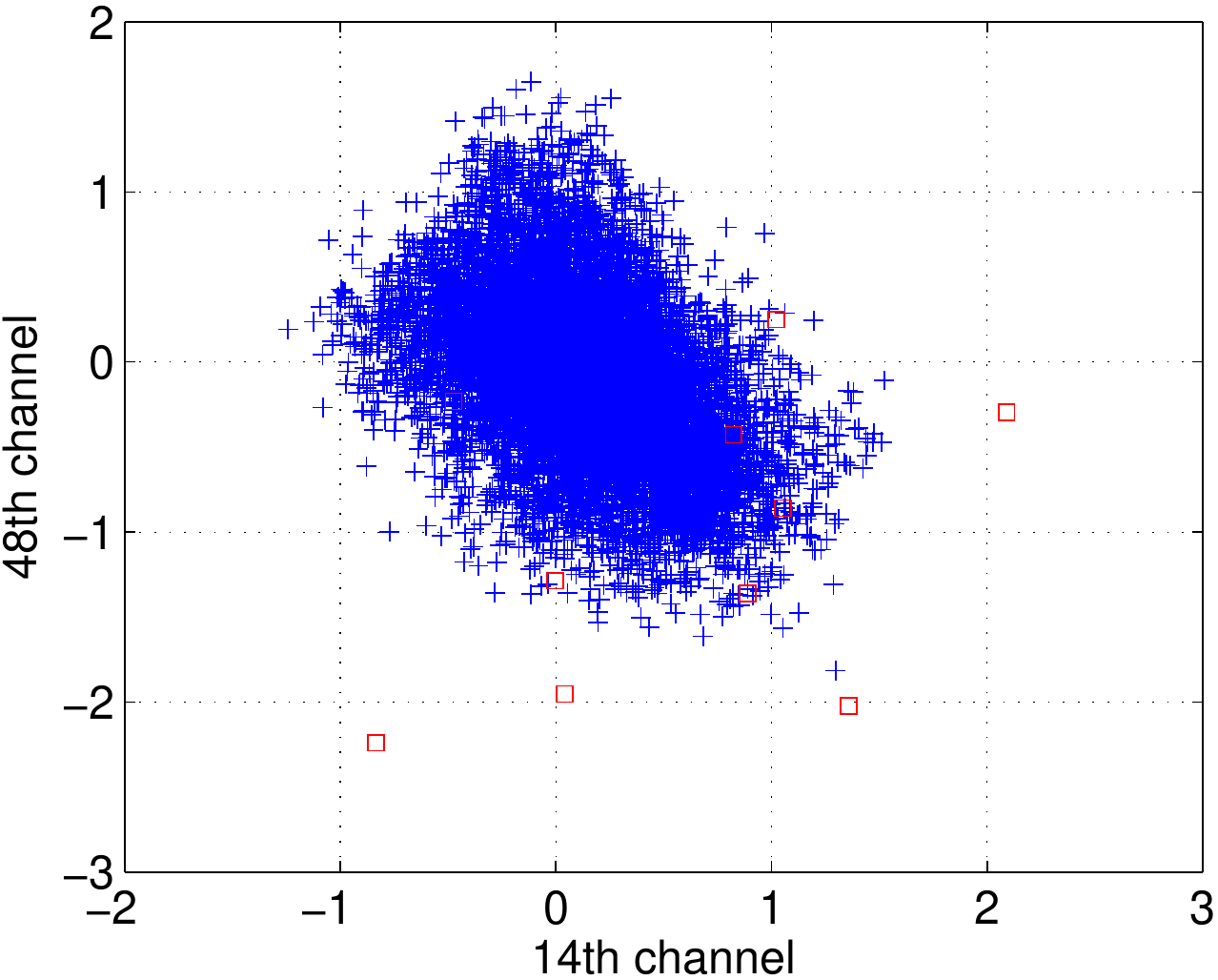}\\
	\end{tabular}
	\caption{Statistics of four feature channels from the 2nd convolutional layer in LeNet-5 on MNIST. Top two rows: histograms of feature values. Bottom two rows: joint distributions of feature values. Note: all feature values are from the center entry (4, 4); red squares are from nonsense samples of pure Gaussian noise.}
	\label{fig:dist_mnist_plainLeNet}
\end{figure}

\begin{figure}[htbp]
	\centering
	\begin{tabular}{ZZ}
		\includegraphics[width=3.0cm]{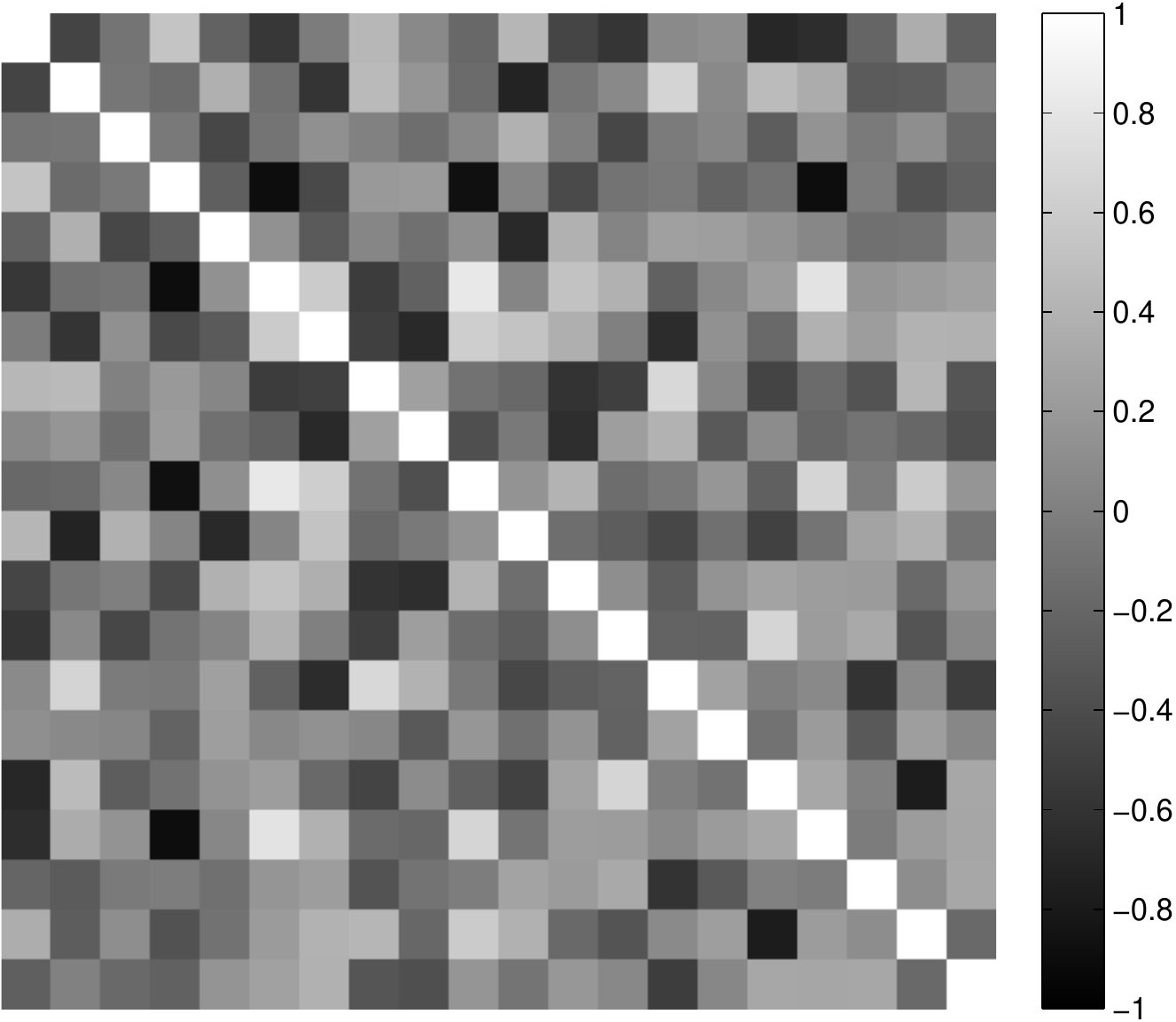}&
		\includegraphics[width=3.0cm]{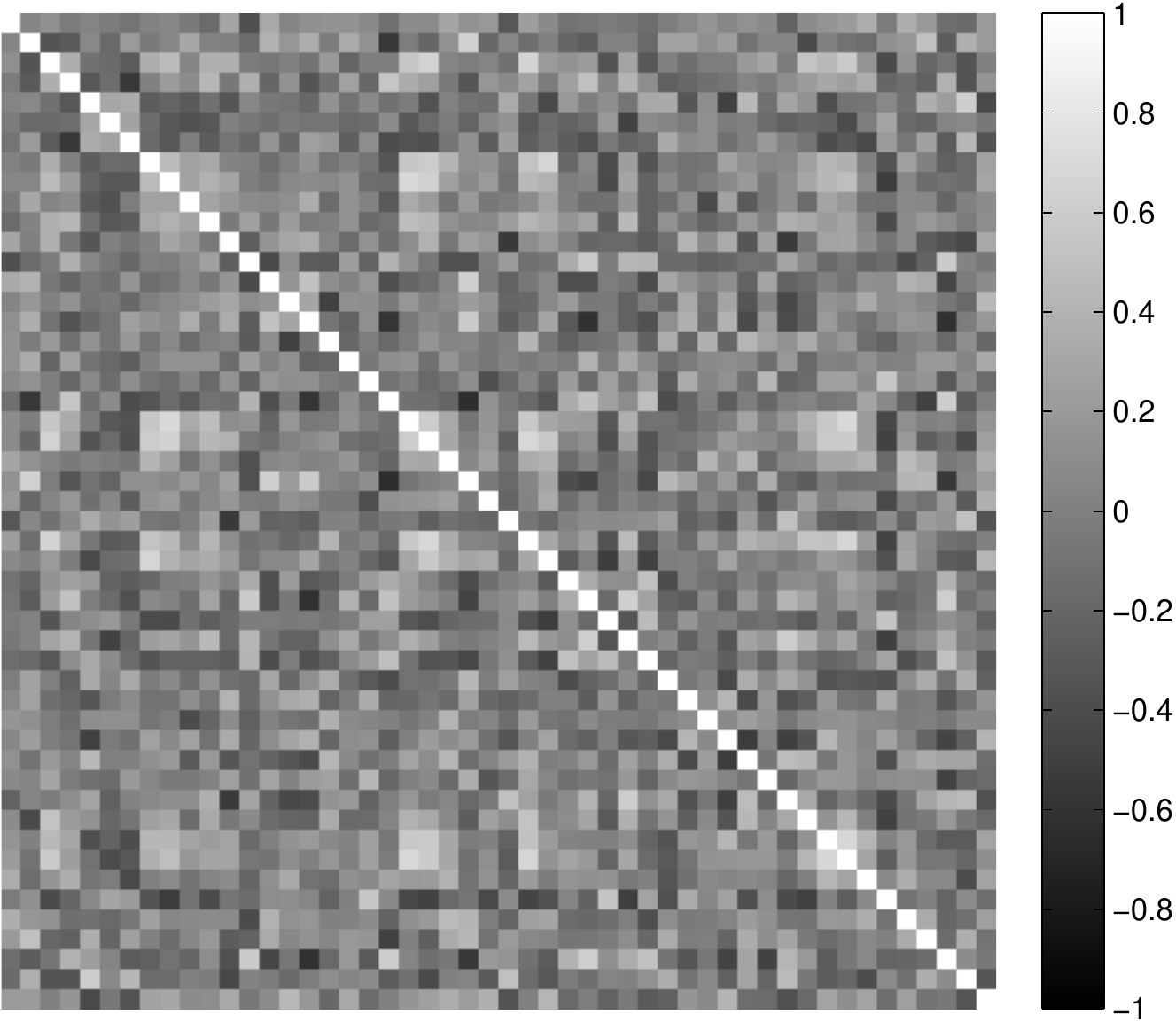}\\
		conv-1 & conv-2 \\
		\includegraphics[width=3.0cm]{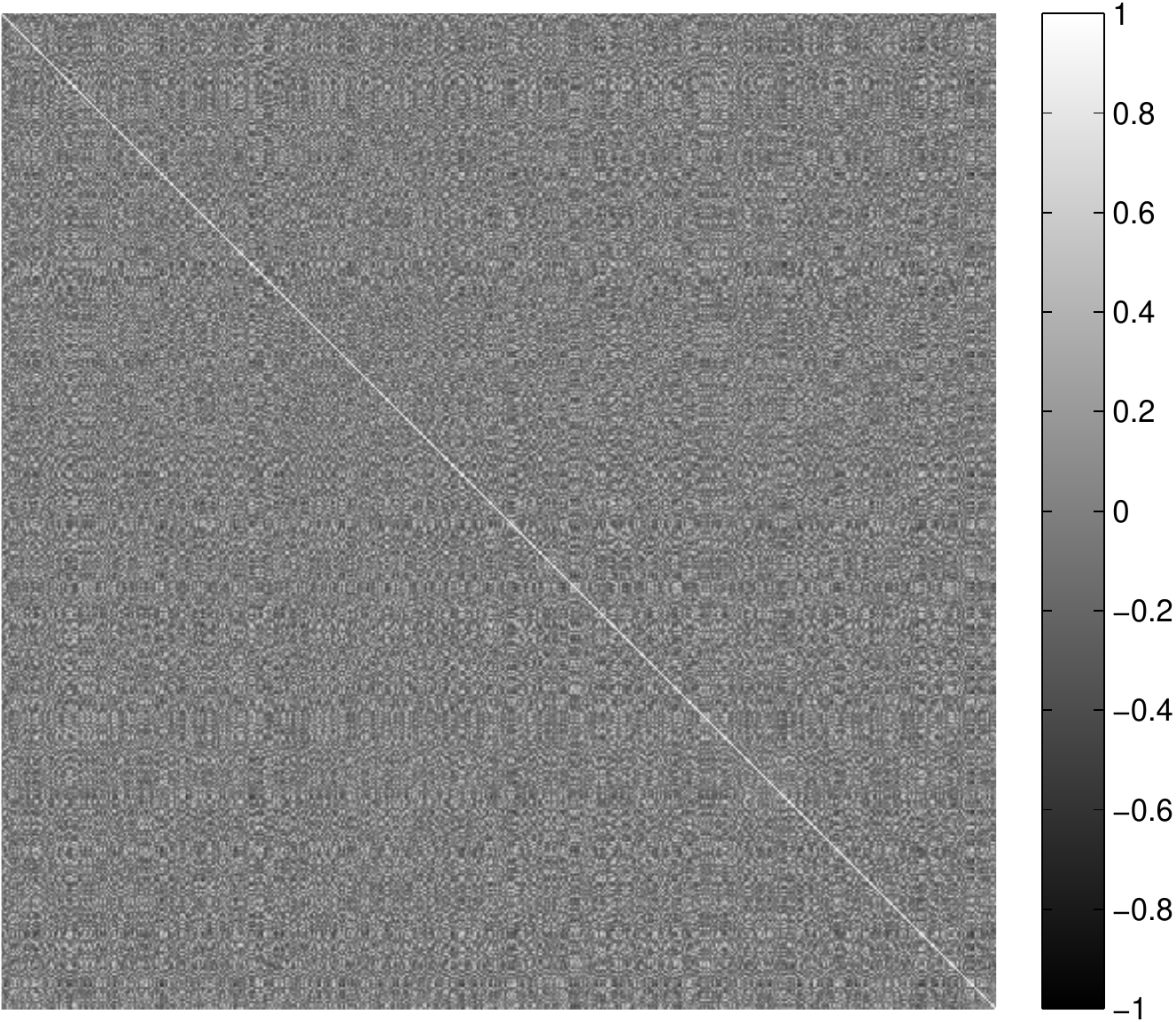}&
		\includegraphics[width=3.0cm]{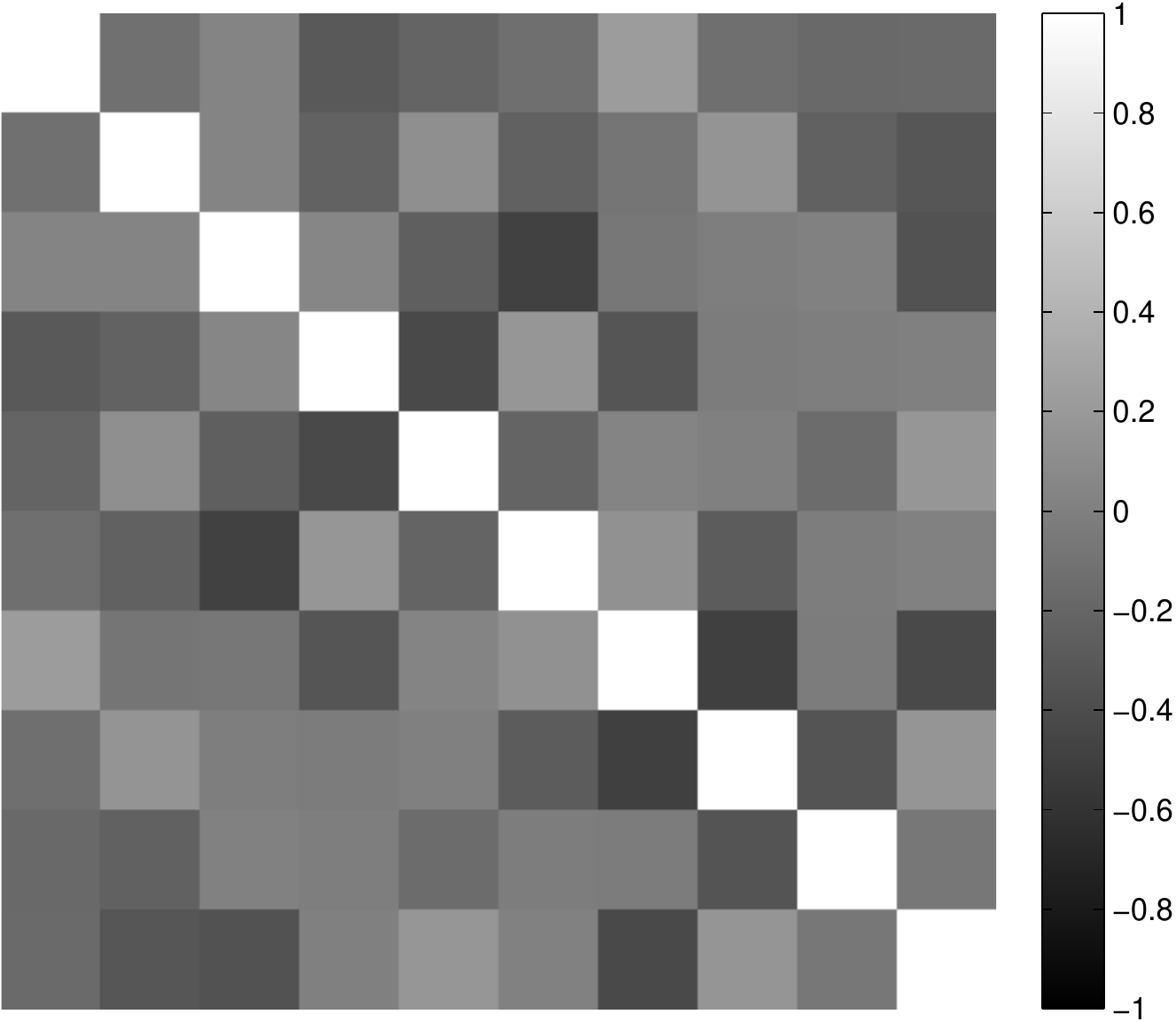}\\
		fc-1 & fc-2 \\
	\end{tabular}
	\caption{Visualization of the correlation matrices of features (on specified entries) from the convolutional/fully-connected layers in LeNet-5 on the MNIST test set.}
	\label{fig:corr_mnist_plainLeNet}
\end{figure}

\section{Symmetric Activation Functions}
\label{sec:saf}

Compared with commonly-used activation functions such as Rectified Linear Units (ReLU) \cite{Nair} and sigmoid, it is well known the RBFs have a localizing property. 1-D RBFs suppress signals that deviate from the normal in either direction while their counterparts suppress signals only unilaterally. By suppressing the signals of abnormal samples, we ensure that no strong prediction confidence will be produced at the top layer. Inspired by the ReLU, we propose the mirrored rectified linear unit (mReLU). In the paper, we call both the 1-D RBF and the mReLU symmetric activation functions (SAF) since they have the reflectional symmetry.

\subsection{1-D Radial Basis Function}
\label{sec:rbf1d}

The 1-D RBF adopted in the paper is $\sigma(x)=e^{-x^2}$, and its derivative is $-2x\sigma(x)$ (Fig. \ref{fig:symfunc_der}.a-b). Although it is parameter-free, any needed flexibility can be achieved by argument rescaling and shifting in the preceding convolutional layer. 

\subsection{Mirrored Rectified Linear Unit (mReLU)}
\label{sec:sReLU}

The 1-D RBF requires exponential and square calculations. These will add considerable load in both the training and test stages. Inspired by the simplicity of the ReLU, we propose the mReLU composed of only basic arithmetic and logic calculations (Fig. \ref{fig:symfunc_der}.c-d):

\begin{equation}
\label{eq:mReLU}
\text{mReLU}(x)=\left\{
\begin{array}{lll}
1+x, &\ \ \ \ {0>x>-1,}\\
1-x, &\ \ \ \ {+1>x>0,}\\
0, &\ \ \ \ \text{otherwise,}
\end{array} \right.
\end{equation}
or equivalently
\begin{equation}
\label{eq:mReLU_ReLUform}
\text{mReLU}(x)=\min\left(\text{ReLU}(1-x),\text{ReLU}(1+x)\right).
\end{equation}
Its derivative is
\begin{equation}
\label{eq:d_s2cases}
\frac{d}{dx}\text{mReLU}=\left\{
\begin{array}{lll}
+1, &\ \ \ \ {0>x>-1;}\\
-1, &\ \ \ \ {+1>x>0;}\\
0, &\ \ \ \ \text{otherwise.}
\end{array} \right.
\end{equation}

\subsection{The Capacity of SAF Networks}

In the limit of using an infinite number of RBF units, a network of high-dimension RBFs can approximate any function with arbitrary precision \cite{Park}. We can use the products of 1-D RBFs to exactly mimic any high-dimensional RBF, and therefore construct a 1-D RBF network for function approximation. However, the \emph{product-of-units} scheme is not compatible with popular CNNs, and it does not work for the mReLUs. Below we establish a property which is weaker than function approximation but with closer ties to our problem: sample classification.

We express the problem in geometric form. Suppose we have $N$ template samples $Z_1,Z_2,...,Z_N$ in a uniform $n$-D data space, and a fidelity threshold $r$. If an input sample $X$ has $\lVert X-Z_i\rVert \le r$ for a certain $i\in \{1,2,...,N\}$, we classify $X$ into the same category as $Z_i$; otherwise we regard it as a \emph{nonsense} sample. We assume the problem is self-consistent, i.e., there are no template samples from different categories with overlapping radius-$r$ hyperspheres. It is sufficient to find a method to judge whether a sample is in the radius-$r$ hypersphere centered at a given sample $Z_i$. First, we have
\begin{lemma}
	\label{lem:approx_hypersphere}
	For a given hypersphere, we can build a network of two SAF layers to judge whether a given sample is in the hypersphere with arbitrary precision.
\end{lemma}
\begin{proof}
	We present the proof for 1-D RBFs, as it is straightforward to be extended to the mReLU. We build a basic block of five layers as shown in Fig. \ref{fig:capacity}.a. Let $X=\left[x_1,x_2,...,x_n\right]^T,Z_i=\left[z_{i,1},z_{i,2},...,z_{i,n}\right]^T$. We calculate
	\begin{equation}
	\label{eq:r_s_t_y}
	\begin{split}
	r_k(X)&=\lambda(x_k-z_{i,k}), \ \ \ s_k(X)=e^{-r_k^2},\\
    t(X)&=\sum_k{s_k}-n, \ \ \ \ \ \ \ \ y(X)=e^{-t^2}
	\end{split}
	\end{equation}
	where $\lambda$ is a carefully chosen parameter, the $s_k$'s are the output of the first SAF layer and the single $y$ is the output of the second SAF layer. To achieve a given error rate $\epsilon$, it is sufficient to find a threshold $\tau$ which holds $\frac{V(\Omega \cap \Omega^\prime)}{V(\Omega \cup \Omega^\prime)}\ge 1-\epsilon$. Here $V(\cdot)$ is the volume of a closed subspace, $\Omega=\{X:\lVert X-Z_i\rVert \le r\}$ and $\Omega^\prime=\{X:y(X) \ge \tau\}$. We let
	\begin{equation}
	\label{eq:f_X_r}
	\begin{split}
	f_1(r)&=y\left(\left[z_{i,1}+r,z_{i,2},...,z_{i,n}\right]^T\right),\\ 
	f_2(r)&=y\left(\left[z_{i,1}+\frac{r}{\sqrt{n}},z_{i,2}+\frac{r}{\sqrt{n}},...,z_{i,n}+\frac{r}{\sqrt{n}}\right]^T\right).
	\end{split}
	\end{equation}
	It is easy to prove that $f_1(r)\le y(X)\le f_2(r)$ for any $X$ which holds $\lVert X - Z_i \rVert = r$.
	
	We set $\tau=f_1(r)$ and $r_0=f_2^{-1}(\tau)$. It is easy to see that $r_0 \le r$. So we have $\Omega^{\prime\prime}\subset\Omega^\prime \subset \Omega$ where $\Omega^{\prime\prime}=\{X:\lVert X-Z_i\rVert \le r_0\}$. Since we have $V(\Omega)=c\cdot r^n$ and $V(\Omega^{\prime\prime})=c\cdot r_0^n$ for a certain constant $c$, it is required that $\left(\frac{r_0}{r}\right)^n\ge 1-\epsilon$, or equivalently
	\begin{equation}
	\label{eq:f2_f1_1_e}
	\frac{f_2^{-1}(\tau)}{f_1^{-1}(\tau)}\ge \left(1-\epsilon\right)^{\frac{1}{n}}.
	\end{equation}
	Since $\lim\limits_{\tau\rightarrow 1}\frac{f_2^{-1}(\tau)}{f_1^{-1}(\tau)}=1$, there must exist a $\tau_0$ satisfying eq. \ref{eq:f2_f1_1_e}, and so does $\lambda$. In other words, we can judge whether a sample $X$ is in the hypersphere of $Z_i$ with a error rate less than $\epsilon$ by comparing $y(X)$ and $\tau_0$.
\end{proof}
\begin{figure}[htbp]
	\centering
	\begin{tabular}{ZZ}
		\includegraphics[width=3.0cm]{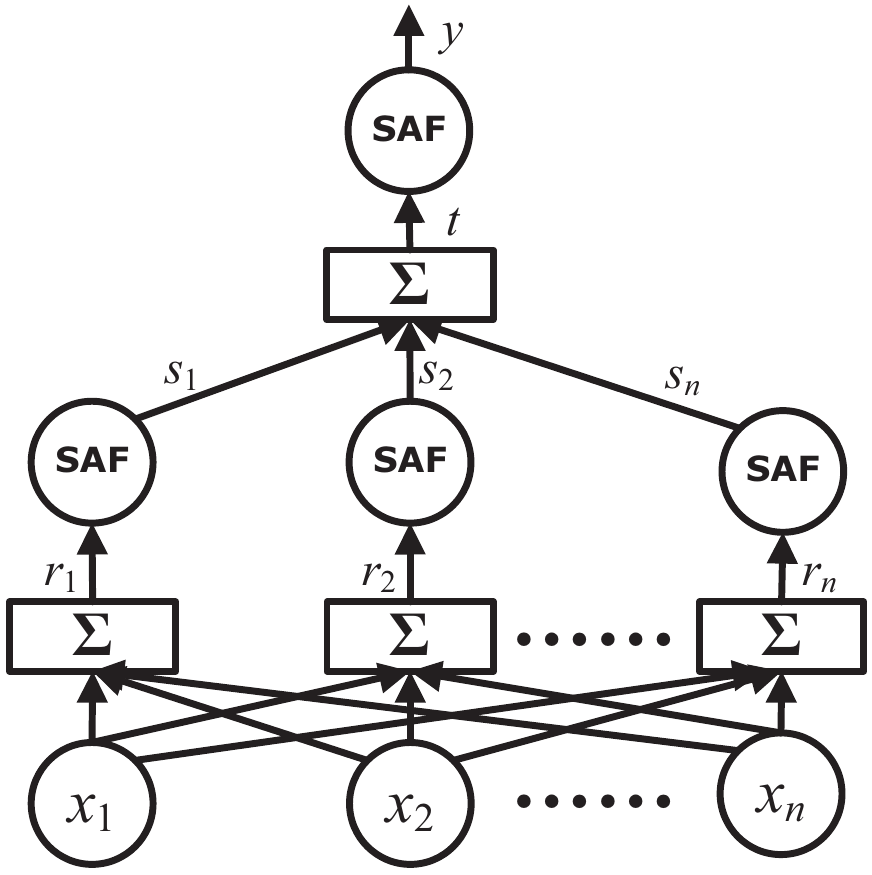}&
		\includegraphics[width=3.5cm]{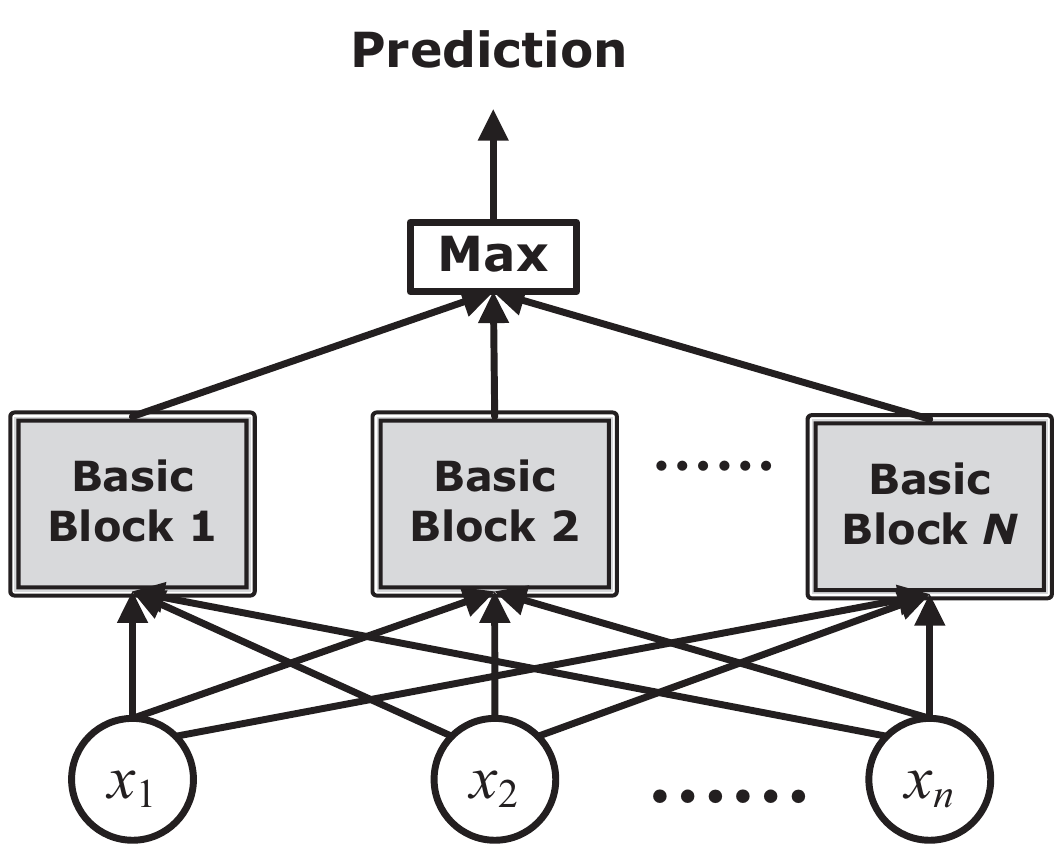} \\
		(a) & (b) \\
	\end{tabular}
	\caption{Using SAF units to build classification networks. (a) a basic block using two SAF layers to judge whether an input is in a $n$-D hypersphere approximately. (b) a network using the basic blocks in (a) to perform the classification task.}
	\label{fig:capacity}
\end{figure}
\begin{figure}
	\centering
	\begin{tabular}{DDD}
		\includegraphics[width=2.2cm]{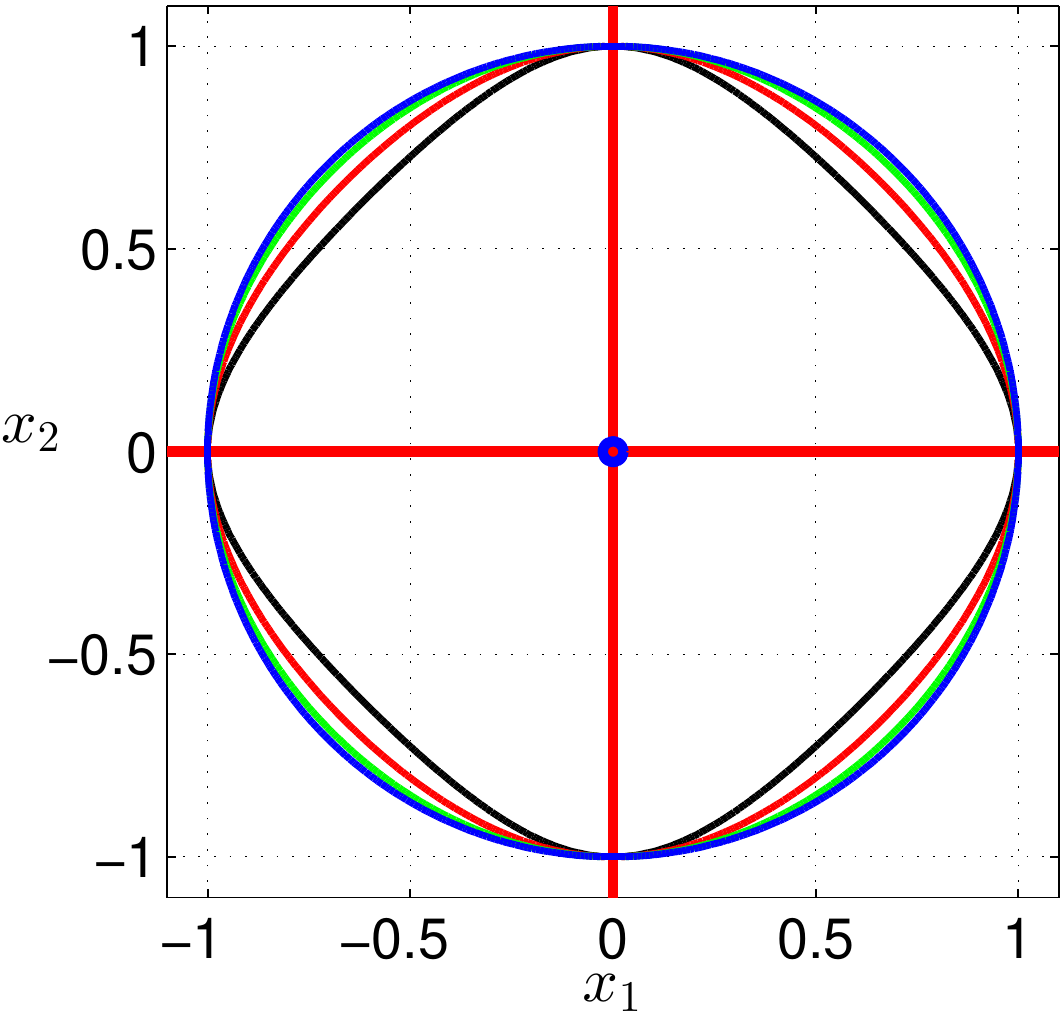}&
		\includegraphics[width=2.5cm]{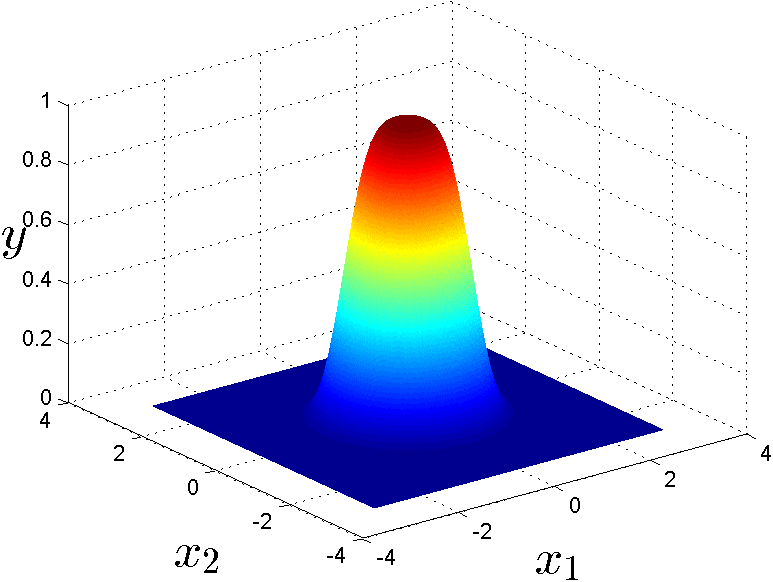}&
		\includegraphics[width=2.5cm]{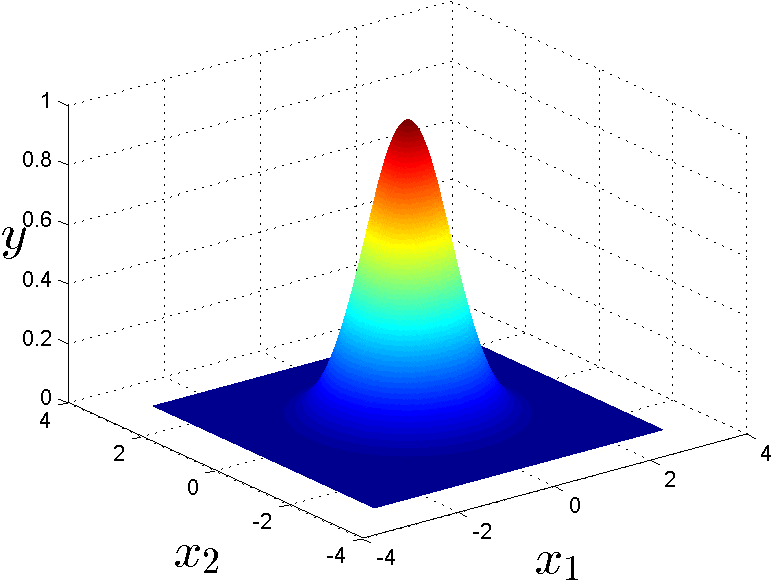}\\
		\ \ \ \ (a) & (b) & (c)
	\end{tabular}
	\caption{2-D approximations using 1-D RBFs. (a) using the basic block to approximate a unit disk with different values of $\lambda$: 0.1 (bordered in blue), 0.4 (green), 0.7 (red) and 1.0 (dark). (b) the approximation in the $x_1$-$x_2$-$y$ space using five 1-D RBF units and $\lambda=0.4$. (c) a 2-D RBF.}
	\label{fig:shapes}
\end{figure}

For example, we can approximate the unit disk with three 1-D RBF units in Fig. \ref{fig:shapes}.a. Then we use the basic block in the proof of Lemma \ref{lem:approx_hypersphere} to build the classification network. As shown in Fig. \ref{fig:capacity}.b, we lay a basic block for each template sample, then choose the maximum $y_{max}$ across all basic blocks. It $y_{max}$ is larger than the threshold $\tau_0$ in the proof of Lemma \ref{lem:approx_hypersphere}, we output the category of the chosen basic block, otherwise we give a \emph{nonsense} prediction. Interestingly, if we are allowed to use more 1-D RBF units in the basic block, the output $y$ will be increasingly similar to a high-dimension RBF, as shown in Fig. \ref{fig:shapes}.b-c.

Although the network in Fig. \ref{fig:capacity}.b is compatible with the popular CNNs in structure, the network is infeasible in practice as we would require a huge number of SAF units. Instead in Sec. \ref{sec:building_robustCNN}, we exploit the advantages of deep CNNs, which generate representations distributedly, to build effective classifiers with moderate number of SAFs.

\section{Building Robust CNNs}
\label{sec:building_robustCNN}

In popular CNNs, convolutional layers are cascaded to learn increasingly abstract features \cite{Zeiler}. We find the empirical distributions of these features are fairly compact and symmetric (Fig. \ref{fig:dist_mnist_plainLeNet}). Thus it is feasible to suppress unusual signals using SAFs which are similarly shaped. We insert additional SAF layers immediately after the convolutional layers to suppress unusual signals, as shown in Fig. \ref{fig:building_RobustCNN}. Any ReLU activation layers can be discarded since the SAF outputs are non-negative. We do the same thing to sigmoid activation layers. We also add a 1-D RBF layer after the fully-connected layer as a new final layer, but do not insert SAF layers between fully-connected layers. In the remainder, we call the modified models \emph{robust CNNs} and the original ones \emph{plain CNNs}.

\subsection{The Hybrid Loss function}

We design a special hybrid loss function for the robust CNN that combines the negative logarithm of the normalized value, and the weighted $p$-order errors ($p>1$) according to:
\begin{equation}
\label{eq:normloss}
\mathcal{L}_{h}\left(X,\ell,\theta\right)=-\alpha_1\cdot\ln \left(\frac{y_\ell}{\sum_i{y_i}}\right)+\alpha_2\cdot\left(1-y_\ell\right)^p+\alpha_3\cdot\sum_{i\ne \ell}{y_i^p}
\end{equation}
where $\ell$ is the correct label, $Y=\{y_i\}$ is the prediction on $X$ by the robust CNN with the parameter $\theta$, $L$ is the amount of categories, and $\alpha_1,\alpha_2, \alpha_3$ are weighting parameters. $\mathcal{L}_h$ is designed to prefer $y_\ell$ to be large both relatively and absolutely. It is necessary to involve the $p$-order errors to get sufficiently large $y_\ell$, thus to recognize nonsense samples which would be of small $y_\ell$'s. We set $p=2$ in all our experiments.

\subsection{Robustness Analysis}
\begin{figure}[htbp]
	\centering
	\begin{tabular}{c}
		\includegraphics[width=7.5cm]{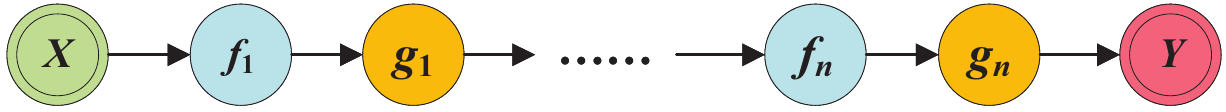} \\
	\end{tabular}
	\caption{A simplified depiction of a robust CNN. $X$ is the input, and $Y$ is the prediction. Blue circles: linear layers. Yellow circles: SAF layers.}
	\label{fig:CNN_model}
\end{figure}
In the simplified robust CNN in Fig. \ref{fig:CNN_model}, $f_i$'s are convolutional or pooling modules, and $g_i$'s are the SAF modules following convolutional layers. We have
\begin{equation}
Y=g_{n}(f_{n}(g_{n-1}(f_{n-1}(...g_{1}(f_{1}(X))...))))
\end{equation}
where $Y=\left[y_1,y_2,..,y_L\right]^T$ is the confidence vector of length $L$, the number of categories. The partial derivative of $y_\ell$ with respect to $x_k$ is
\begin{equation}
\label{eq:partial_xk}
\frac{\partial y_\ell}{\partial x_k}=\frac{d g_{n,\ell}}{d f_{n,\ell}}\cdot\sum_{i_{n-1},...,i_1}{\left(\frac{\partial f_{n,\ell}}{\partial g_{n-1,i_{n-1}}}\cdot...\cdot\frac{\partial f_{1,i_{1}}}{\partial x_k}\right)}
\end{equation}
where $i_1\!\rightarrow\! i_2\!\rightarrow,...,\rightarrow\! i_{n-1}\!\rightarrow\!\ell$ is a path from $x_k$ to $y_\ell$ crossing all layers, in other words, they are the indices of cells whose receptive field contain $x_k$. For unusual signals which are away from the high-density regions, the $g$'s in eq. \ref{eq:partial_xk} will be considerably small. Consequently, according to the derivatives of SAFs shown in Fig. \ref{fig:symfunc_der}, the $\frac{\partial g_{\left(.,.\right)}}{\partial f_{\left(.,.\right)}}$'s are also small or nearly zero, implying $\frac{\partial y_\ell}{\partial x_k}$ is unlikely to be large. If the norm of the gradient $\nabla{y_\ell}$ with respect to $X$ is small, we have to make a large perturbation $\Delta X$ to achieve a noticeable change of $y_\ell$, as shown in Fig. \ref{fig:robust_analysis}. In conclusion, to achieve an adversarial sample successfully, the perturbation for the robust CNN would be larger than that for the plain CNN.

\begin{figure}
	\centering
	\begin{tabular}{c}
		\includegraphics[width=8.0cm]{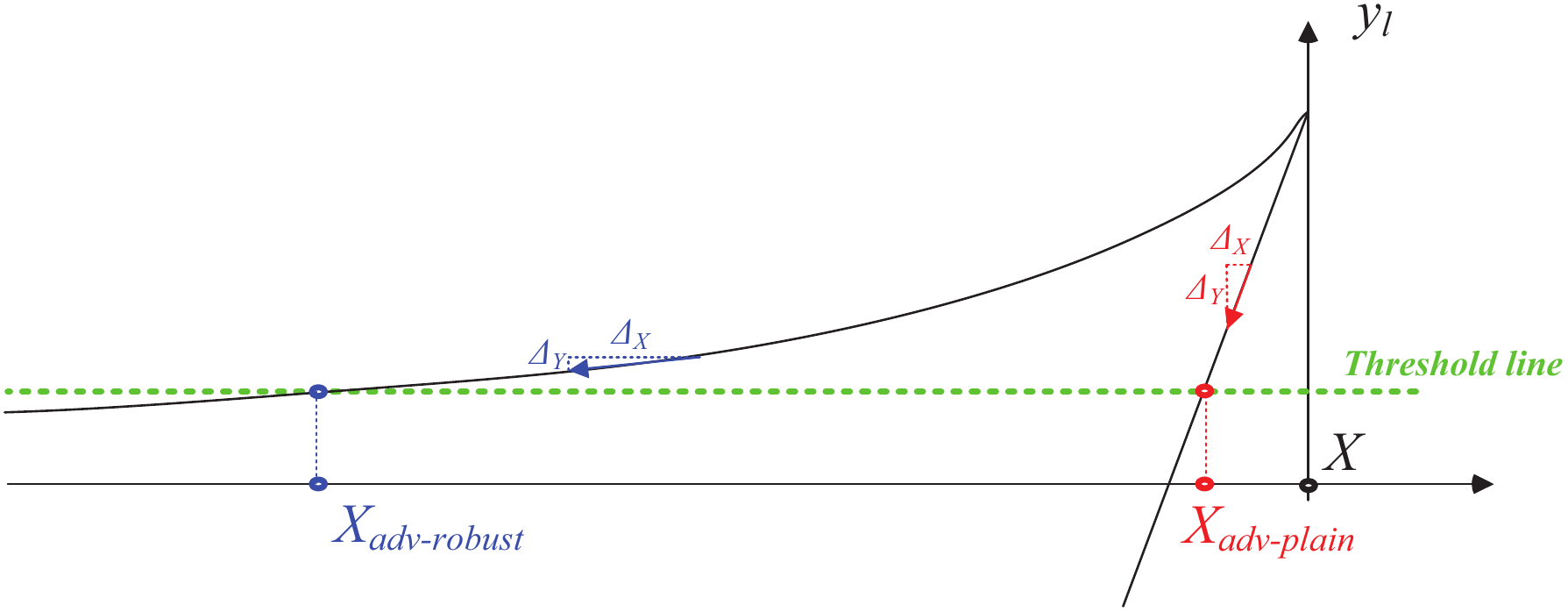} \\
	\end{tabular}
	\caption{Perturbing samples for the plain and robust CNNs. $X$ is a clean sample, and $y_\ell$ is the confidence score for the $\ell$th category. $X_{adv-robust}$ and $X_{adv-plain}$ are the adversarial samples for the robust and plain CNNs respectively.}
	\label{fig:robust_analysis}
\end{figure}

\subsection{Training Issues}
\label{sec:training}

An adversarial training requires repeated generations of adversarial samples and so increases the training load significantly. In this paper, we also use training data additional to the clean samples; but instead of adversarial samples, we use random perturbation of the mean of training data. We find the mean sample itself might be easily mis-classified (it seldom belongs to any meaningful category), and we can generate nonsense samples by perturbing it slightly. Therefore we use it as a typical hard instance, and train robust CNNs to produce high confidence for \emph{nonsense} on it. However, there is only one mean sample and thus little effect on the training process. So we perturb it with the Gaussian noise and get a certain number of noisy copies. We call this trick \emph{mean training} for short.

It is suggested in \cite{Goodfellow} that randomly perturbed samples are of little significance. Interestingly, we do find them effective for our network, especially when adopted together with mean training. We generate perturbed samples by adding perturbations randomly drawn from $\left[-\delta,+\delta\right]^{w\times h \times c}$ ($w$: width, $h$: height, $c$: channel) to the original samples. We perturb a fraction of the train set and leave others untouched. We call this trick \emph{random training} for short.

We use the standard stochastic gradient descent method to train the robust CNNs. It is unnecessary to use a pre-training stage to find RBF centers because 1-D RBFs are parameter-free. The batch normalization trick \cite{Ioffe} is necessary to accelerate training robust CNNs and prevent
the gradient from vanishing.

\section{Experiments}
\label{sec:exp}

We use two well-studied datasets, MNIST \cite{LeCun} and CIFAR-10 \cite{CIFAR}, in our experiments. In the following figures and tables, we use \emph{plain}, \emph{RBF} and \emph{mReLU} to refer respectively to plain CNNs, robust CNNs using 1-D RBF, and robust CNNs using mReLUs. We use the flags -a, -r and -m to indicate the use of adversarial, random and mean training. The plain CNNs are typical models implemented in \cite{Vedaldi}. We use all default settings in \cite{Vedaldi} to train them, except for whitening and normalization. We plan to release our source code and data upon acceptance.

In related literature, CNNs are expected to be robust against three kinds of sample: adversarial, nonsense and noisy. If our proposals about SAFs are correct, then robust CNNs will be as accurate as plain CNNs; on adversarial/noisy samples with tiny perturbations, the classification accuracy should not drop remarkably; while on severely perturbed samples or pure noise images, they should be classified as \emph{nonsense}. The nonsense case is a little different for models without an explicit nonsense category: we consider they predict a sample as nonsense if all confidence scores of meaningful categories are below a threshold, e.g. 0.5.

We follow \cite{Goodfellow} to generate adversarial samples and nonsense samples by:
\begin{equation}
\label{eq:s_adv}
X_{adv}=X+255\cdot\beta\cdot \text{sign}\left(\nabla \mathcal{L}_h(X,\ell_{max},\theta)\right)
\end{equation}
where $\nabla \mathcal{L}_h(X,\ell,\theta)$ is the gradient towards the smaller loss for a certain incorrect category $\ell$, $\beta$ is the strength parameter, and $\ell_{max}=\arg\max_\ell{\lVert\nabla \mathcal{L}_h(X,\ell,\theta)\rVert_1}$. $X$ and $X_{adv}$ are the original and adversarial/nonsense samples respectively. Nonsense samples are generated from the stationary Gaussian noise, shifted and scaled to cover the range 0$\sim$255. Noisy samples are generated by perturbing clean samples with stationary Gaussian noise. We collect classification accuracies on varying perturbation strength $\beta$, rather than at a single strength value, to better understand the robustness of CNNs. Meanwhile, we present the well-adopted Peak-Signal-Noise-Ratios (PSNR) to assess the image quality easily.

\subsection{MNIST}
\label{sec:exp_mnist}

We use LeNet-5 \cite{LeCun} as the plain CNN for MNIST, please see Table \ref{tab:str_mnist} for its structure. There are 70,000 clean samples and we follow \cite{Vedaldi} to use 60,000 as the train set and 10,000 as the test set. There are 10,000 groups of nonsense samples. All CNNs are trained with 20 epochs. We set $\alpha_1\!=\!1,\alpha_2\!=\!1,\alpha_3\!=\!0$ for the hybrid loss functions of robust CNNs.

\begin{table*}[htbp]\footnotesize
	\begin{center} 
		\caption{Structures of the plain and robust CNNs for MNIST. Parameters of convolutional layers: \textbf{cv1}-(5, 5, 20), \textbf{cv2}-(5, 5, 50) (in the \emph{height-width-channel} order). The number of hidden units in fully-connected layers are 500 (\textbf{fc1}) and 10 (\textbf{fc2}). \textbf{max}-max pooling. \textbf{sloss}-softmaxloss. \textbf{hloss}-hrbridloss.}
		\label{tab:str_mnist}
		\begin{tabular}{l|c|c|c|c|c|c|c|c|c|c|c|c} 
			\hline
			Models & \multicolumn{11}{c|}{Layers} & \#Layers \\
			\hline
			Plain & cv1 & - & max &  cv2 & - & max & fc1 & ReLU & fc2 & - & sloss & 8 \\
			\hline
			RBF & cv1 & 1-D RBF & max &  cv2 & 1-D RBF & max & fc1 & ReLU & fc2 & 1-D RBF & hloss & 11 \\
			\hline
			mReLU & cv1 & mReLU & max &  cv2 & mReLU & max & fc1 & ReLU & fc2 & 1-D RBF & hloss & 11 \\
			\hline
		\end{tabular} 
	\end{center} 
\end{table*}

We present accuracies and error rates in Fig. \ref{fig:mnist_robustness} and Table \ref{tab:mnist_robustness}. It is clear that CNNs using SAFs are much more robust than the plain CNN. The best one, mReLU-r-m, makes marginally more errors than the plain CNN for clean samples, adversarial samples with perturbation of up to strength $\beta=0.01$, and noisy samples of up to $\beta=0.10$; but has greatly superior performance for adversarial and noisy samples of large perturbation, and all nonsense samples. Particularly, the accuracy of mReLU-r-m under adversarial perturbations of the strength $\beta=0.25$ is only 0.073, which is much lower than 0.179 of Maxout \cite{Goodfellow2013} trained using adversarial training \cite{Goodfellow}.

Random training is very effective in improving the robustness of SAF CNNs against both adversarial and noisy samples. In our opinion, random perturbations of the train samples make the large-output/non-zero widths of SAFs effectively broader, which improving robustness to perturbations of moderate strength. Using all training tricks, we can improve all CNNs to be considerably robust against nonsense and noisy samples, even the plain CNN. Adversarial training, which is recommended in previous literatures, is much less effective improving the robustness against both adversarial and nonsense samples. It is probably due to our insufficient number of adversarial training rounds. So in Sec. \ref{sec:cifar}, we train all CNNs with only the combination of random and mean training.

\begin{figure*}[htbp]
	\centering
	\begin{tabular}{BBX}
		\includegraphics[width=4.8cm]{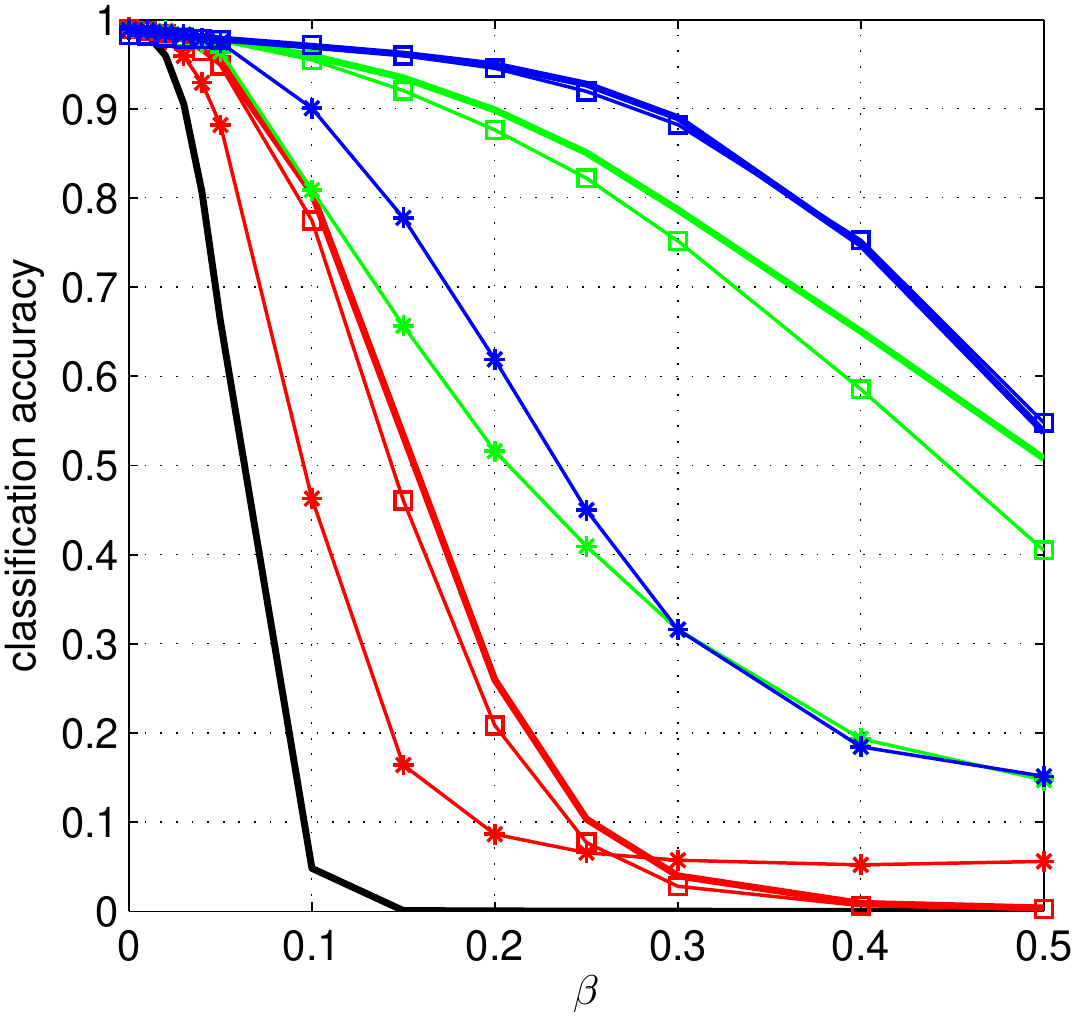} &
		\includegraphics[width=4.8cm]{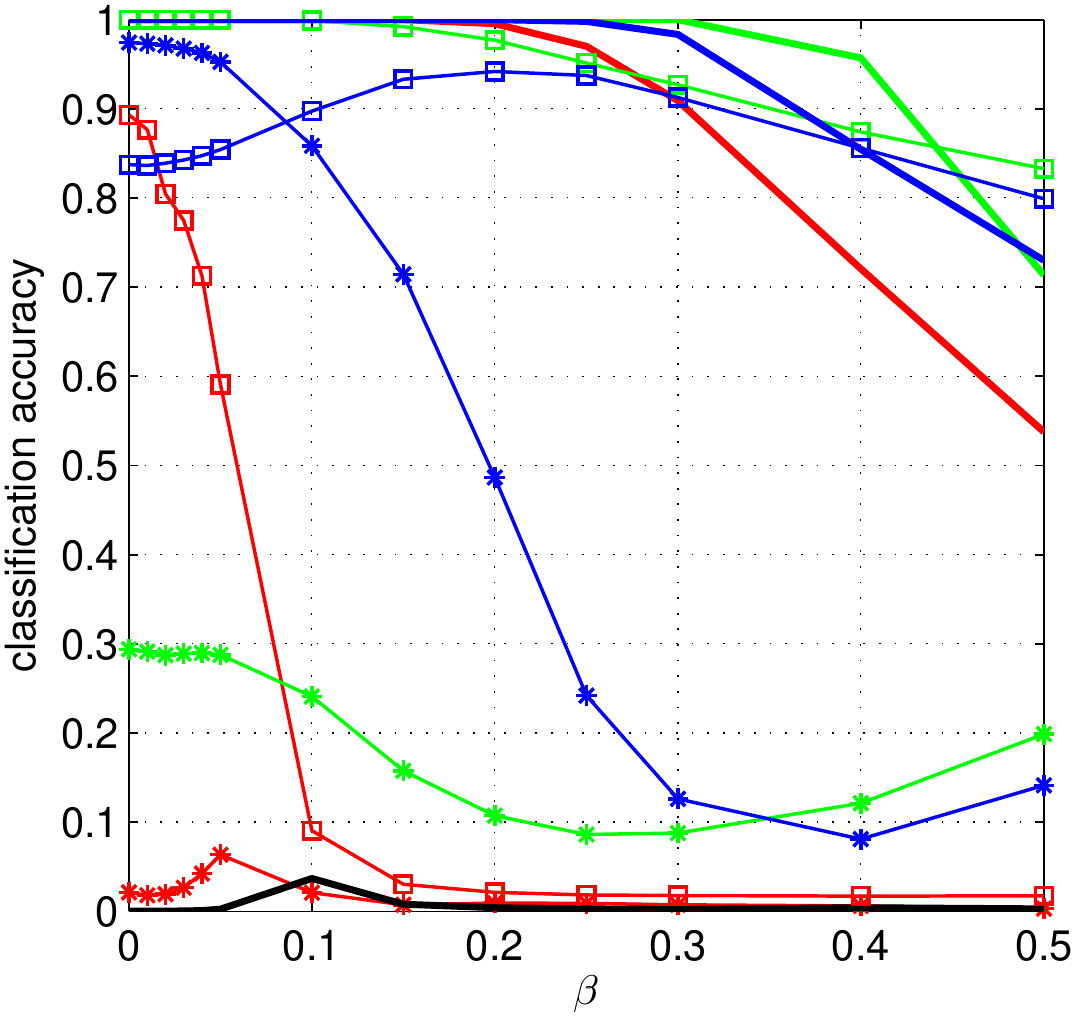} &
		\includegraphics[width=6.5cm]{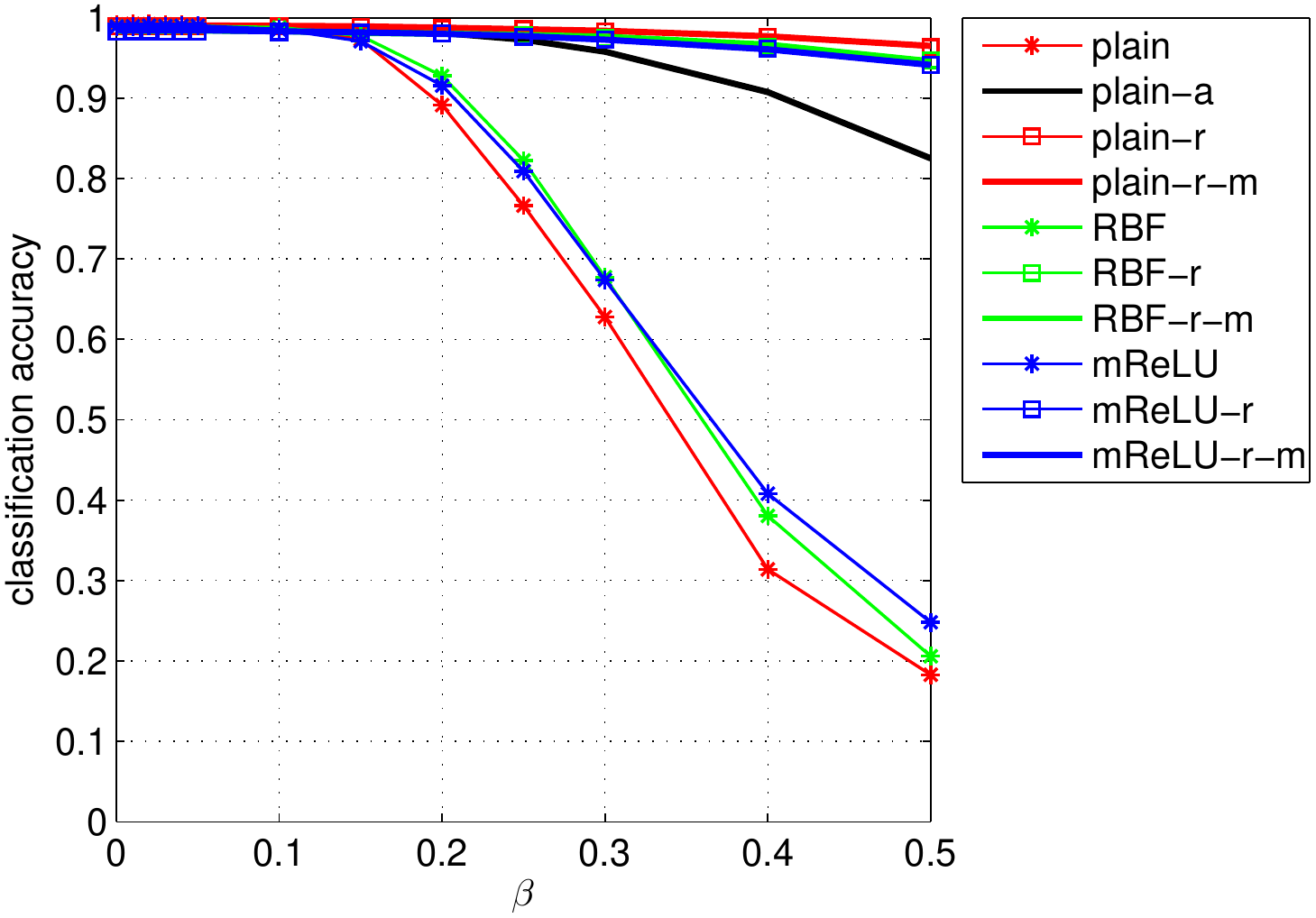} \\
		\ \ \ (a) & \ \ \ (b) & \multicolumn{1}{l}{\ \ \ \ \ \ \ \ \ \ \ \ \ \ \ (c)}
	\end{tabular}
	\caption{Accuracies of the CNNs on the MNIST test set. (a), (b) and (c) are the accuracies on adversarial, nonsense and noisy samples respectively. The horizontal axes are the perturbation strength $\beta$'s.}
	\label{fig:mnist_robustness}
\end{figure*}

We show some examples of mReLU-r-m classifications in Fig. \ref{fig:mnist_examples}. The robust CNN gives correct predictions on even severe distorted samples. It does make incorrect predictions for samples bounded by color boxes, but the images are hard even for a human being. The case of green boxes is a bit different: the predictions are wrong, but it is sensible to classify these samples into the \emph{nonsense} category.

\begin{table*}\footnotesize
	\begin{center} 
		\caption{Error rates on the MNIST test set. The best values are shown in bold.}
		\label{tab:mnist_robustness}
		\begin{tabular}{l|c|r|r|r|r|r|r|r|r|r|r|r|r} 
			\hline
			$\beta$ & clean & 0.01 & 0.02 &  0.03 & 0.04 & 0.05 & 0.10 & 0.15 & 0.20 & 0.25 & 0.30 & 0.40 & 0.50 \\
			\hline
			PSNR & - & 40.00 & 33.98 &  30.46 & 27.96 & 26.02 & 20.00 & 16.48 & 13.98 & 12.04 & 10.46 & 7.96 & 6.02 \\
			\hline
			\multicolumn{14}{l}{} \\ [-2.00 ex]
			\multicolumn{14}{l}{On \textbf{adversarial} samples:} \\
			\multicolumn{14}{l}{} \\ [-2.00 ex]
			\hline
			plain & \textbf{0.009} & \textbf{0.013} & 0.023 & 0.041 & 0.071 & 0.118 & 0.537 & 0.836 & 0.914 & 0.934 & 0.943 & 0.948 & 0.944 \\
			\hline
			mReLU-r-m & 0.015 & 0.016 & \textbf{0.017} & \textbf{0.019} & \textbf{0.020} & \textbf{0.022} & \textbf{0.030} & \textbf{0.039} & \textbf{0.051} & \textbf{0.073} & \textbf{0.110} & \textbf{0.252} & \textbf{0.463} \\
			\hline
			\multicolumn{14}{l}{} \\ [-2.00 ex]
			\multicolumn{14}{l}{On \textbf{nonsense} samples:} \\
			\multicolumn{14}{l}{} \\ [-2.00 ex]
			\hline
			plain & 0.979 & 0.982 & 0.981 & 0.974 & 0.958 & 0.937 & 0.979 & 0.993    & 0.991 & 0.991 & 0.993 & 0.994 & 0.997 \\
			\hline
			mReLU-r-m & \textbf{0.000}  &  \textbf{0.000}  &  \textbf{0.000}  &  \textbf{0.000} &  \textbf{0.000}  &  \textbf{0.000} &  \textbf{0.000} &  \textbf{0.000} &  \textbf{0.000} & \textbf{0.002} & \textbf{0.017} & \textbf{0.146} & \textbf{0.270} \\
			\hline
			\multicolumn{14}{l}{} \\ [-2.00 ex]
			\multicolumn{14}{l}{On \textbf{noisy} samples:} \\
			\multicolumn{14}{l}{} \\ [-2.00 ex]
			\hline
			plain & \textbf{0.009} & \textbf{0.008} & \textbf{0.008} & \textbf{0.009} & \textbf{0.009} & \textbf{0.010} & \textbf{0.013} & 0.027 & 0.109 & 0.234 & 0.372 & 0.686 & 0.817 \\
			\hline
			mReLU-r-m & 0.015 & 0.016 & 0.016 & 0.016 & 0.016 & 0.016 & 0.017 & \textbf{0.018} & \textbf{0.020} & \textbf{0.023} & \textbf{0.027} & \textbf{0.039} & \textbf{0.059} \\
			\hline
		\end{tabular} 
	\end{center} 
\end{table*}

\begin{figure*}\footnotesize
	\centering
	\begin{tabular}{crrrrrrrrrrrc}
		\multicolumn{13}{c}{\includegraphics[width=15.5cm]{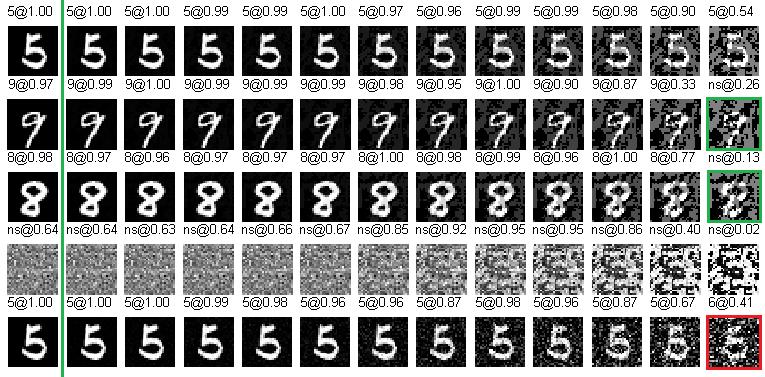}}\\
		\hline
		\multicolumn{13}{l}{Strengths $\beta$}\\
		\ \ clean & \ \ \ \ 0.01 & \ \ \ 0.02 &  \ \ \ 0.03 & \ \ \ 0.04 & \ \ \ \ 0.05 & \ \ \ 0.10 & \ \ \ 0.15 & \ \ \ 0.20 & \ \ 0.25 & \ \ 0.30 & \ \ \ 0.40 & 0.50 \\
		\hline
		\multicolumn{13}{l}{PSNR}\\
		- & 40.00 & 33.98 &  30.46 & 27.96 & 26.02 & 20.00 & 16.48 & 13.98 & 12.04 & 10.46 & 7.96 & 6.02 \\
		\hline
	\end{tabular}
	\caption{Examples of mReLU-r-m classifications on the MNIST test set. The prediction and confidence are on the top of individual samples. \textbf{ns}-nonsense. Rows 1-3: adversarial samples. 4th row: nonsense samples. 5th row: noisy samples.}
	\label{fig:mnist_examples}
\end{figure*}

\subsection{CIFAR-10}
\label{sec:cifar}

The plain CNN for CIFAR-10 is a deeper LeNet model; see Table \ref{tab:str_cifar} for its structure. There are 60,000 clean samples and we follow \cite{Vedaldi} to use 50,000 as the train set and 10,000 as the test set. The training set is augmented by randomly reflecting samples about their vertical mid-line. There are 10,000 groups of nonsense samples. The robust CNNs are trained with 90 epochs, twice of that for the plain CNN. We set $\alpha_1\!=\!1,\alpha_2\!=\!1,\alpha_3\!=\!\frac{1}{9}$ for the hybrid loss functions of robust CNNs.

\begin{table*}[htbp]\scriptsize
	\begin{center} 
		\caption{Structures of the plain and robust CNNs for CIFAR-10. Parameters of convolutional layers: \textbf{cv1}-(5, 5, 32), \textbf{cv2}-(5, 5, 32), \textbf{cv3}-(5, 5, 64) (in the \emph{height-width-channel} order). The number of hidden units in fully-connected layers are 64 (\textbf{fc1}) and 10 (\textbf{fc2}). \textbf{max}-max pooling. \textbf{avg}-average pooling. \textbf{sloss}-softmaxloss. \textbf{hloss}-hrbridloss.}
		\label{tab:str_cifar}
		\begin{tabular}{l|c|c|c|c|c|c|c|c|c|c|c|c|c|c|c} 
			\hline
			Models & \multicolumn{14}{c|}{Layers} & \#Layers \\
			\hline
			Plain & cv1 & - & max &  cv2 & - & avg & cv3 & - & avg & fc1 & ReLU & fc2 & - & sloss & 10 \\
			\hline
			RBF & cv1 & 1-D RBF & max &  cv2 & 1-D RBF & avg & cv3 & 1-D RBF & avg & fc1 & ReLU & fc2 & 1-D RBF & hloss & 14 \\
			\hline
			mReLU & cv1 & mReLU & max &  cv2 & mReLU & avg &  cv3 & mReLU & avg & fc1 & ReLU & fc2 & 1-D RBF & hloss & 14 \\
			\hline
		\end{tabular} 
	\end{center} 
\end{table*}

We present the accuracies and error rates in Fig. \ref{fig:cifar_robustness} and Table \ref{tab:cifar_robustness}. Since the CIFAR-10 images are colored, the perturbation is much easier to perceive. Therefore we use perturbations of smaller magnitude than those for the MNIST. As with the MNIST, the robust CNNs are much more robust than the plain CNN. The best one, mReLU-r-m, makes marginally more errors than the plain CNN for clean samples, and noisy samples of up to $\beta=0.015$; but has greatly superior performance for noisy samples of large perturbation, all adversarial samples and all nonsense samples. However, the improvement on CIFAR-10 is not as large as that on the MNIST. The reason is that there are not sufficient training samples in the CIFAR-10 to train a perfect classifier, thus lots of samples are close to the decision boundaries and easy to perturb to be an adversary. The robustness of plain-r-m against nonsense samples is not as good as with the MNIST, and some curves in Fig. \ref{fig:cifar_robustness}.b seem a bit odd. We leave it to the future work.

Interestingly, the accuracy of mReLU CNN (without random and mean training) on clean samples is 0.194 which is marginally better than that of the plain CNN, 0.217. Although we believe the plain CNN can perform better with advanced training tricks, we hypothesize that there is a connection between the robustness and generalization capability. Thus the CNN using mReLUs can generalize slightly better on the test set.

We show some examples of mReLU-r-m classifications in Fig. \ref{fig:cifar_examples}. The model makes correct predictions on most samples, but makes incorrect predictions for the samples in the red squares which are of the severest perturbation.

\begin{figure*}[htbp]
	\centering
	\begin{tabular}{BBX}
		\includegraphics[width=4.8cm]{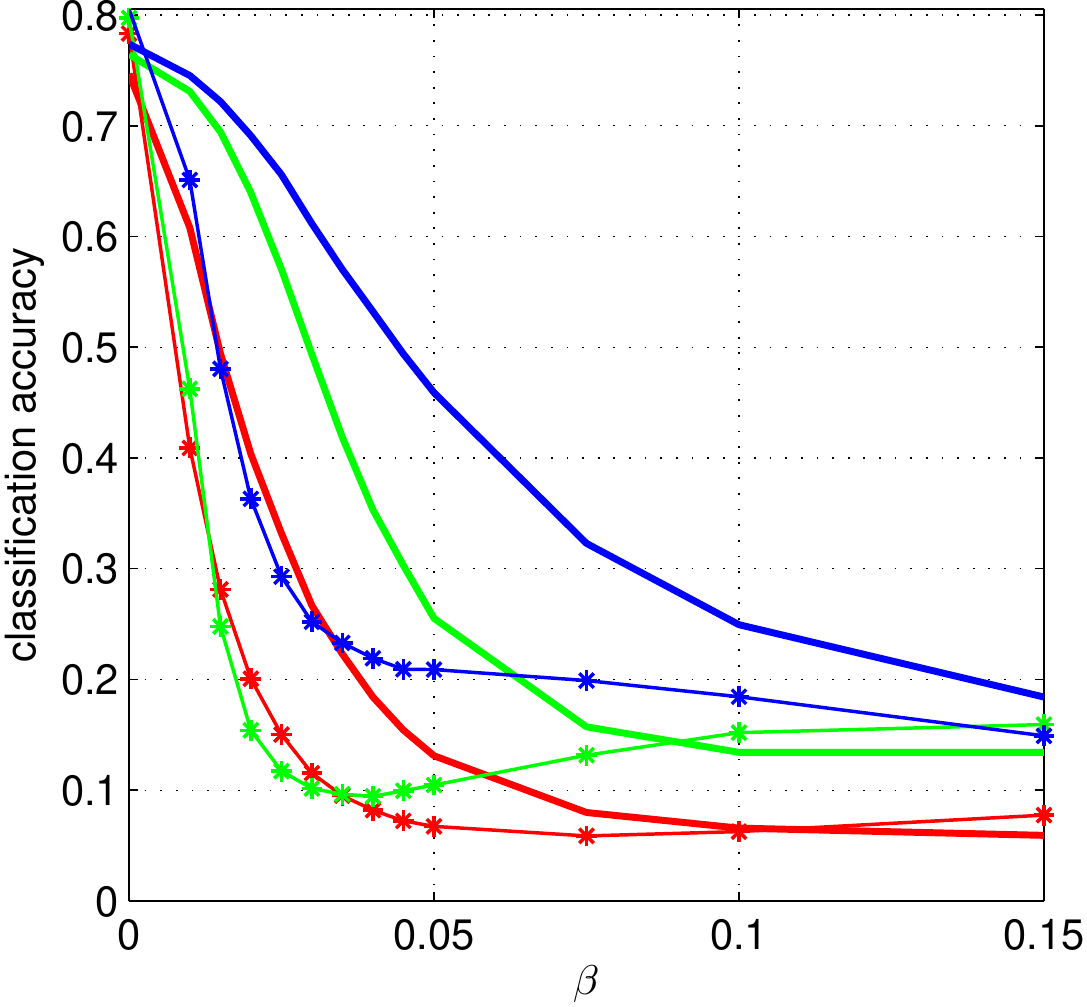} &
		\includegraphics[width=4.8cm]{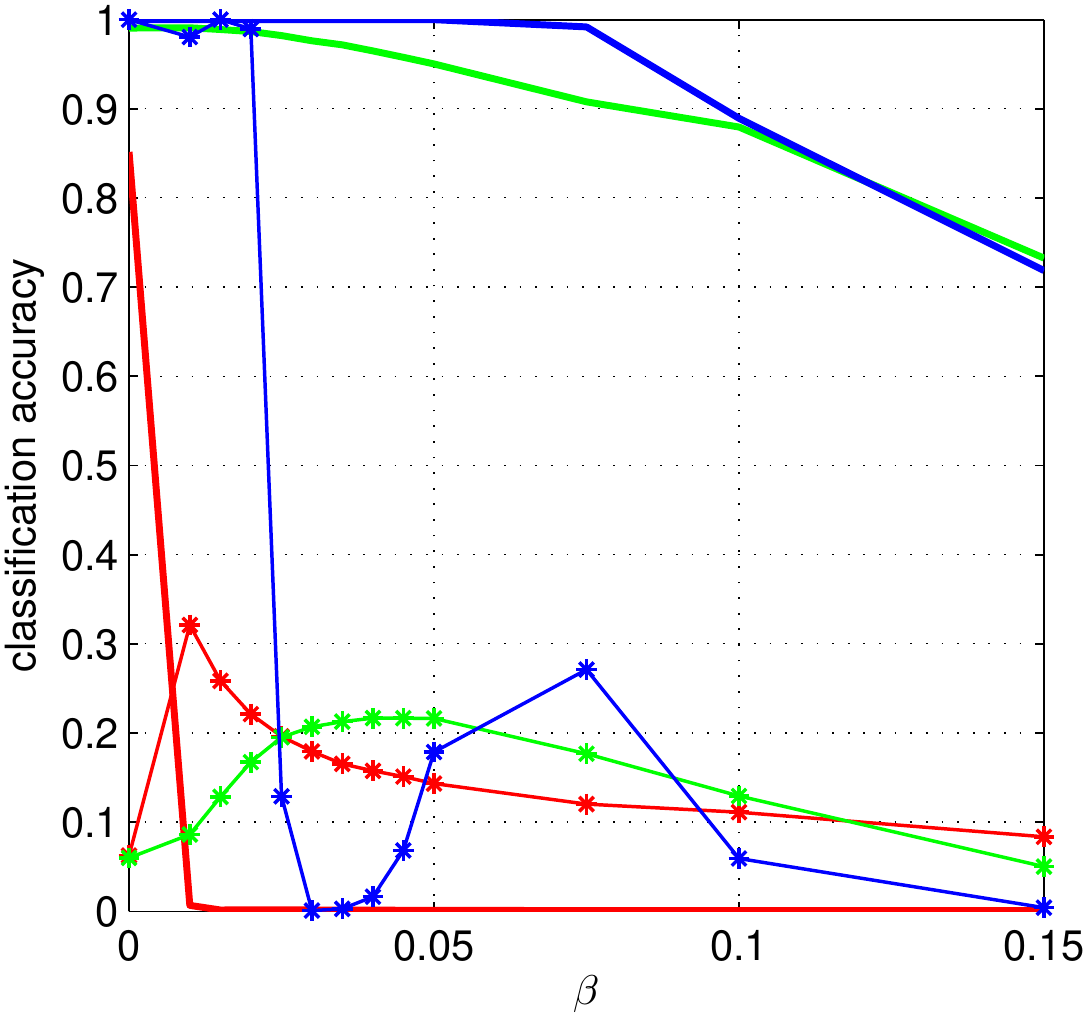} &
		\includegraphics[width=6.5cm]{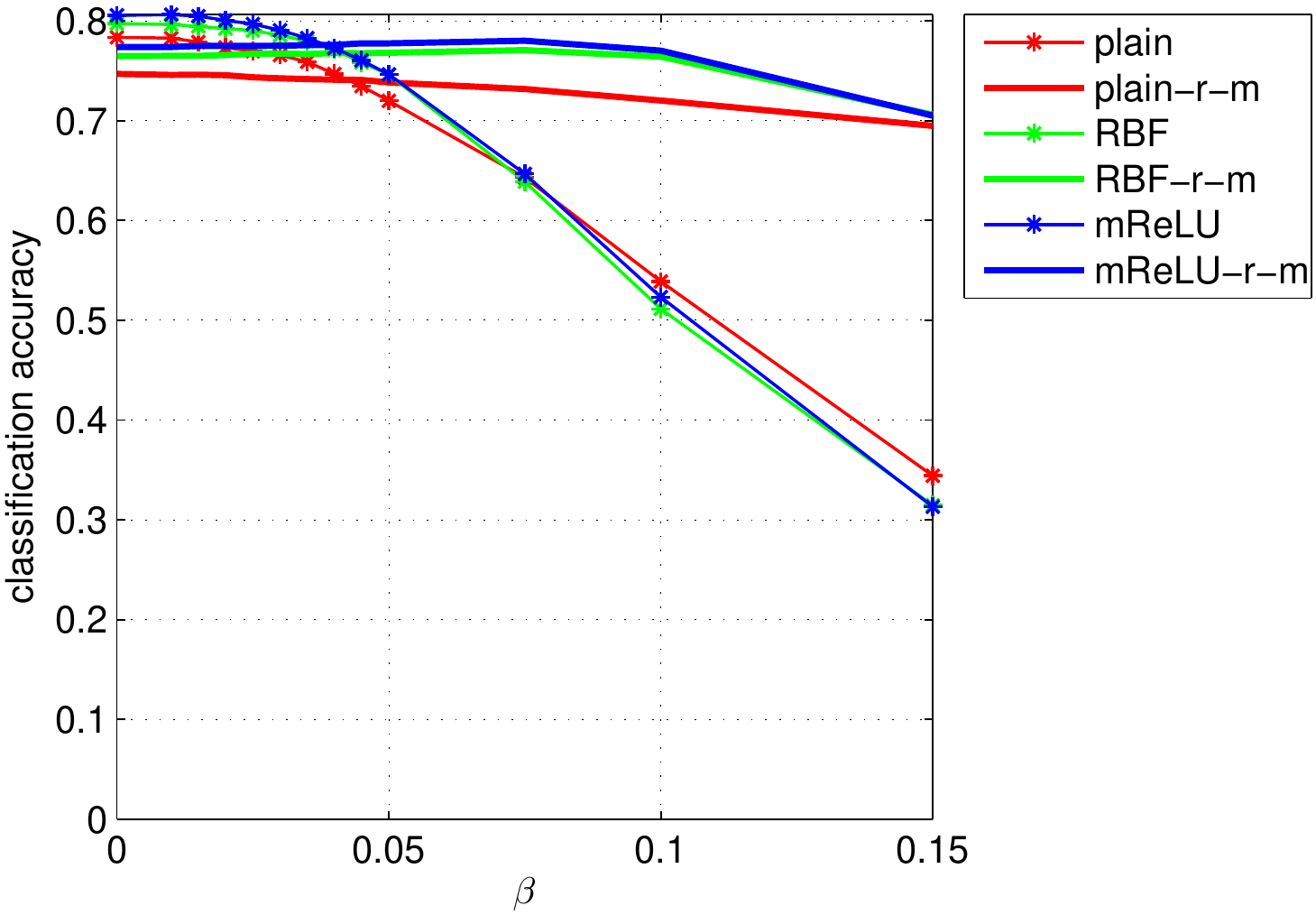} \\
		\ \ \ (a) & \ \ \ (b) & \multicolumn{1}{l}{\ \ \ \ \ \ \ \ \ \ \ \ \ \ \ (c)}
	\end{tabular}
	\caption{Accuracies of the CNNs on the CIFAR-10 test set. (a), (b) and (c) are the accuracies on adversarial, nonsense and noisy samples respectively. The horizontal axes are the perturbation strength $\beta$'s.}
	\label{fig:cifar_robustness}
\end{figure*}

\begin{table*}[htbp]\footnotesize
	\begin{center} 
		\caption{Error rates on the CIFAR-10 test set. The best values are shown in bold.}
		\label{tab:cifar_robustness}
		\begin{tabular}{l|c|r|r|r|r|r|r|r|r|r|r|r|r} 
			\hline
			$\beta$ & clean & 0.010 & 0.015 & 0.020 & 0.025 & 0.030 & 0.035 & 0.040 & 0.045 & 0.050 & 0.075 & 0.100 & 0.150 \\
			\hline
			PSNR & - & 40.00 & 36.48 & 33.98 & 32.04 & 30.46 & 29.12 & 27.96 & 26.94 & 26.02 & 22.50 & 20.00 & 16.48 \\
			\hline
			\multicolumn{14}{l}{} \\ [-2.00 ex]
			\multicolumn{14}{l}{On \textbf{adversarial} samples:} \\
			\multicolumn{14}{l}{} \\ [-2.00 ex]
			\hline
			plain &  \textbf{0.217} & 0.591 & 0.719 & 0.800 & 0.850 & 0.885 & 0.905 & 0.918 & 0.928 & 0.933 & 0.941 & 0.938 & 0.922 \\
			\hline
			mReLU-r-m & 0.226 & \textbf{0.255} & \textbf{0.278} & \textbf{0.309} & \textbf{0.344} & \textbf{0.388} & \textbf{0.430} & \textbf{0.468} & \textbf{0.506} & \textbf{0.541} & \textbf{0.677} & \textbf{0.751} & \textbf{0.816} \\
			\hline
			\multicolumn{14}{l}{} \\ [-2.00 ex]
			\multicolumn{14}{l}{On \textbf{nonsense} samples:} \\
			\multicolumn{14}{l}{} \\ [-2.00 ex]
			\hline
			plain & 0.937 & 0.679 & 0.742 & 0.779 & 0.804 & 0.821 & 0.835 & 0.843 & 0.849 & 0.857 & 0.880 & 0.889 & 0.916 \\
			\hline
			mReLU-r-m & \textbf{0.000} & \textbf{0.000} & \textbf{0.000} & \textbf{0.000} & \textbf{0.000} & \textbf{0.000} & \textbf{0.000} & \textbf{0.000} & \textbf{0.000} & \textbf{0.000} & \textbf{0.008} & \textbf{0.111} & \textbf{0.281} \\
			\hline
			\multicolumn{14}{l}{} \\ [-2.00 ex]
			\multicolumn{14}{l}{On \textbf{noisy} samples:} \\
			\multicolumn{14}{l}{} \\ [-2.00 ex]
			\hline
			plain & \textbf{0.217} & \textbf{0.218} & \textbf{0.221} & 0.227 & 0.231 & 0.234 & 0.241 & 0.253 & 0.266 & 0.280 & 0.357 & 0.461 & 0.656 \\
			\hline
			mReLU-r-m & 0.226 & 0.226 & 0.225 & \textbf{0.225} & \textbf{0.225} & \textbf{0.225} & \textbf{0.225} & \textbf{0.224} & \textbf{0.223} & \textbf{0.223} & \textbf{0.220} & \textbf{0.230} & \textbf{0.295} \\
			\hline
		\end{tabular} 
	\end{center} 
\end{table*}

\begin{figure*}[htbp]\footnotesize
	\centering
	\begin{tabular}{crrrrrrrrrrrc}
		\multicolumn{13}{l}{\includegraphics[width=16.5cm]{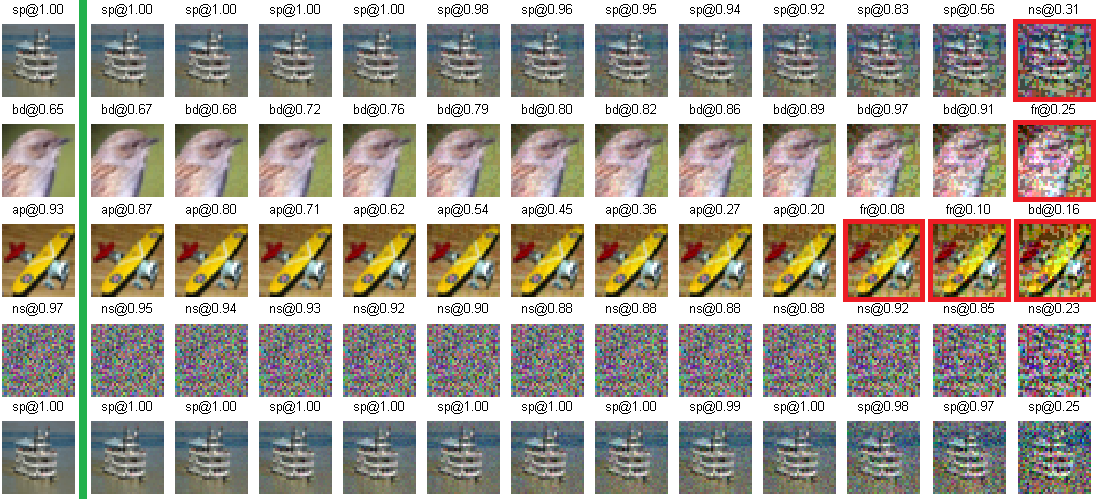}}\\
		\hline
		\multicolumn{13}{l}{Strengths $\beta$}\\
		\ \ \ \ clean & \ \ 0.010 & \ \ 0.015 & \ \ 0.020 & \ \ 0.025 & \ 0.030 & \ \ 0.035 & \ \ 0.040 & \ \ 0.045 & \ 0.050 & \ 0.075 & \ \ 0.100 & \ 0.150 \\
		\hline
		\multicolumn{13}{l}{PSNR}\\
		- & 40.00 & 36.48 & 33.98 & 32.04 & 30.46 & 29.12 & 27.96 & 26.94 & 26.02 & 22.50 & 20.00 & 16.48 \\
		\hline
	\end{tabular}
	\caption{Examples of mReLU-r-m classifications on the CIFAR-10 test set. The prediction and confidence are on the top of individual samples. \textbf{sp}-ship. \textbf{bd}-bird. \textbf{fr}-frog. \textbf{ap}-airplane. \textbf{ns}-nonsense. Rows 1-3: adversarial samples. 4th row: nonsense samples. 5th row: noisy samples.}
	\label{fig:cifar_examples}
\end{figure*}

\section{Conclusions and Discussion}
Our experiments show it is effective to use SAFs to improve the robustness of CNNs. Without substantial changes to the accuracies on clean samples, we obtain remarkably better robustness against adversarial, nonsense and noisy samples. That supports our proposal that SAFs can improve the robustness of CNNs by suppressing unusual signals.

Since we change only the activation functions of CNNs, while leaving their structure, training strategies and optimization methods untouched, it implies there are no severe weaknesses in existing CNN frameworks. Nevertheless, in contrast with popular activation functions including sigmoid and ReLU, the SAFs have less documented support from neuroscience research. Furthermore, it remains a question how to accommodate SAFs better by fine-tuning the CNN structure and training strategies to get the better performance. In that sense, we believe the attempt in this paper is just a start and far from the end.

\nocite{langley00}

\bibliography{robustcnn}
\bibliographystyle{icml2016}

\end{document}